\documentclass[12pt]{article}
\usepackage{lmodern}
\usepackage{parskip}

\usepackage[bottom]{footmisc}
\usepackage{cancel}
\usepackage{url}
\usepackage{fancyhdr}
\usepackage{microtype}
\usepackage{subcaption}
\usepackage{dirtytalk}
\usepackage{booktabs}
\usepackage{float}
\usepackage{multirow}
\usepackage{multicol}
\usepackage{enumerate}
\usepackage{amsmath}
\usepackage{amssymb}
\usepackage{epigraph}
\usepackage{graphicx}
\usepackage{amsthm}
\usepackage{amsfonts}
\usepackage{mathtools}
\usepackage[nice]{nicefrac}
\usepackage{algorithm}
\usepackage{algorithmic}
\usepackage{xcolor}
\usepackage{framed}
\usepackage{bbm}
\usepackage{comment}
\usepackage{thmtools, thm-restate}
\usepackage[toc,page,header]{appendix}
\usepackage[margin=1.25in]{geometry}
\usepackage{times}
\usepackage[round, sort&compress]{natbib}
\definecolor{green}{HTML}{549D54}

\usepackage{hyperref}
\hypersetup{colorlinks,linkcolor=blue}

\usepackage[capitalize,noabbrev,nameinlink]{cleveref}

\title{On the Statistical Query Complexity of Learning Semiautomata: a Random Walk Approach}

\author{
  George Giapitzakis$^{1,}$\thanks{Corresponding author: \texttt{ggiapitz@uwaterloo.ca}} \quad
  Kimon Fountoulakis$^1$ \quad
  Eshaan Nichani$^2$ \\
  Jason D. Lee$^3$\\
  \vspace{-5mm}\\
  \normalsize{$^1$University of Waterloo}\,\,\,
  \normalsize{$^2$Princeton University}\,\,\,
  \normalsize{$^3$University of California, Berkeley}\\
}

\date{\today}

\theoremstyle{plain}
\newtheorem{theorem}{Theorem}[section]

\newtheorem{lemma}{Lemma}[section]
\newtheorem{fact}{Fact}[section]
\newtheorem{corollary}{Corollary}[section]
\theoremstyle{definition}
\newtheorem{definition}{Definition}[section]

\theoremstyle{plain}
\newtheorem{remark}{Remark}[section]

\newcommand{\RR}{\ensuremath{\mathbb{R}}}
\newcommand{\CC}{\ensuremath{\mathbb{C}}}

\newcommand{\indic}{\mathbbm{1}}
\newcommand{\E}{\mathbb{E}}
\newcommand{\Prob}{\mathbb{P}}
\newcommand{\Q}{\mathcal{Q}}
\newcommand{\F}{\mathcal{F}}
\newcommand{\X}{\mathcal{X}}
\newcommand{\Y}{\mathcal{Y}}
\newcommand{\w}{\mathrm{w}}
\newcommand{\C}{\mathcal{C}}

\newcommand{\vect}[1]{\boldsymbol{#1}}

\newcommand\norm[1]{\left\Vert#1\right\Vert_{2}}

\newcommand{\Sn}{S_N}
\newcommand{\U}{\mathcal{U}}
\DeclareMathOperator{\id}{id}
\DeclareMathOperator{\fix}{fix}
\DeclareMathOperator{\GL}{GL}
\DeclareMathOperator{\Tr}{Tr}
\DeclareMathOperator{\triv}{triv}
\DeclareMathOperator{\std}{std}
\DeclareMathOperator{\perm}{perm}
\DeclareMathOperator{\SQDim}{SQDim}

\newcommand{\NN}{\mathbb{N}}

\newcommand{\TDFA}{T_{\operatorname{SA}}}

\newtheorem{example}{Example}

\begin{document}

\maketitle

\begin{abstract}
Semiautomata form a rich class of sequence-processing algorithms with applications in natural language processing, robotics, computational biology, and data mining. We establish the first Statistical Query hardness result for semiautomata under the uniform distribution over input words and initial states.  We show that Statistical Query hardness can be established when both the alphabet size and input length are polynomial in the number of states. Unlike the case of deterministic finite automata, where hardness typically arises through the hardness of the language they recognize (e.g., parity), our result is derived solely from the internal state-transition structure of semiautomata. Our analysis reduces the task of distinguishing the final states of two semiautomata to studying the behavior of a random walk on the group $S_{N} \times S_{N}$. By applying tools from Fourier analysis and the representation theory of the symmetric group, we obtain tight spectral gap bounds, demonstrating that after a polynomial number of steps in the number of states, distinct semiautomata become nearly uncorrelated, yielding the desired hardness result.
\end{abstract}

\section{Introduction}

We study the computational complexity of learning semiautomata with $N$ states in the Statistical Query (SQ) model.
Automata are among the most basic models of computation and arise across formal methods \citep{clarke1999model, isberner2015foundations, holzmann2003spin}, natural language processing \citep{mohri1997finite, maletti2016survey,10.5555/1623516.1623536,10.1017/S135132499600126X}, robotics and control systems \citep{kress2009temporal, chatzilygeroudis2018survey}, computational biology \citep{sakakibara2005grammatical, tsafnat2011computational, de2010grammatical}, and data mining \citep{laxman2006survey}. On the learning-theoretic side, the SQ model captures a broad class of noise-tolerant algorithms \citep{kearns1994introduction}, and its modern relevance is highlighted by its connections to the power of gradient-based methods \citep{abbe2021power}. 

\textbf{Motivation: structural vs. language-based hardness.} 
We investigate hardness results for semiautomata in the SQ model. Prior work on SQ hardness has focused primarily on deterministic finite automata (DFAs), that is, semiautomata with a designated initial state and a set of accepting states. While this distinction may seem subtle at first, it is crucial from a learnability perspective. For DFAs, hardness arises from the complexity of the languages they recognize. Typically, SQ hardness is established by embedding the computation of a hard function, such as parity \citep{blum1994weakly}, into a (random) DFA. Moreover, hardness results often rely on an adversarial input distribution \citep{10.1007/978-3-642-16108-7_18}, instead of the more natural uniform distribution. In contrast, semiautomata are purely state-transition systems. As a result, SQ hardness for semiautomata must come exclusively from the structure of the transitions themselves, namely, from families of transition functions whose induced state-evolution statistics remain pairwise close under the uniform distribution over the input and initial state.

\textbf{Contributions.}
We construct a family of $N!$ semiautomata with $N$ states and an alphabet of size $\Omega(N^3\ln N)$ which are nearly uncorrelated after processing inputs of length $\Omega(N^2 \ln N)$, yielding a statistical dimension of $N!$. Consequently, any SQ learner for that class under the uniform distribution over the input and initial state must either (i) make a super-polynomial number of queries or (ii) use a super-polynomially small tolerance. Our analysis introduces a representation-theoretic toolkit for automata learning: we tie semiautomata processing of random words to the mixing of a random walk on the product of the symmetric group $S_N\times S_N$, and identify a specific irreducible representation that controls indistinguishability.

We propose a randomized construction such that any pair of semiautomata from a randomly generated family of size $N!$ is indistinguishable with high probability under the uniform distribution over input words and starting states. In our random construction, the letters of the alphabet are matched to transpositions and assigned to each semiautomaton by fair coin flips. We reduce the problem of distinguishing two semiautomata that read the same random input to tracking a single coupled random walk on $S_N \times S_N$. In our construction, the probability that two semiautomata agree after starting from the same random initial state and processing $T$ symbols has the form
\[
P_{\operatorname{agree}}(T)=\frac{1}{N}+\operatorname{error}(T,N)
\quad\text{with}\quad
|\operatorname{error}(T,N)|\le \left(1 - \frac{1}{2N}\right)^{T},
\]
with probability at least $1 - \exp(-N\ln N)$. Thus, the absolute error in the agreement beyond the baseline $1/N$ decays exponentially fast as a function of $T$. Taking $T\geq \mathcal{O}(N^2\ln N)$ drives the absolute error to $1/N!$, so any two semiautomata are statistically indistinguishable on polynomial-length inputs. Because we build a family of $N!$ nearly uncorrelated semiautomata, the statistical dimension is $N!$, yielding the desired SQ lower bound. Unlike previous hardness results for DFAs, which rely on reductions to hard function classes such as parity \citep{10.1007/978-3-642-16108-7_18}, our hardness result is inherently structural. We model the behavior of semiautomata processing a random word as a random walk on the group $S_N \times S_N$ and analyze its mixing properties by building off the representation-theoretic framework for random walks on $S_N$ developed by \cite{diaconis}.

Finally, as a result of independent interest, we prove that the mixing time for the random walk of our random family construction is tight. In particular, the best-case mixing time for our random construction is $T=\Omega(N^2\ln N)$. In other words, any smaller choice of $T$ cannot guarantee the indistinguishability bound we obtain. This follows naturally as a byproduct of our proof strategy for the main result.

\textbf{Organization.} The rest of the paper is organized as follows. In \Cref{sec:lit_review}, we provide a review of the relevant literature on the learnability of automata. \Cref{sec:prelim} establishes the notation and necessary mathematical preliminaries. The core of our technical approach is detailed in \Cref{sec:setting}, where we connect the problem of distinguishing semiautomata to the mixing properties of a random walk on the product group $S_N \times S_N$. In \Cref{sec:randomized_family}, we leverage this connection to present a randomized construction of a large family of nearly uncorrelated semiautomata. Finally, in \Cref{sec:sq_hardness}, we use this construction to establish a high Statistical Query dimension, formally proving our main hardness result.
\section{Related work}
\label{sec:lit_review}

Learnability of automata is a foundational problem in computational learning theory, with a rich history spanning several decades \citep{trakhtenbrot1973finite, 10.1145/174644.174647, 10.1007/978-3-642-16108-7_18, pmlr-v291-wang25a}. To the best of our knowledge, no learnability results are currently known for the class of semiautomata. In contrast, the literature on DFA learnability is extensive. In this section, we restrict our review to results on DFA learning within the SQ model. For related work outside the SQ model, we refer the reader to \Cref{app:add_lit}. 

\subsection{The Statistical Query model}
The SQ model, introduced by \citet{kearns1998efficient}, formalizes a large and natural class of algorithms that learn by requesting statistical properties of the data distribution rather than inspecting individual examples. A key motivation for the model is its inherent robustness to noise; any SQ-learnable class is also PAC-learnable \citep{kearns1998efficient}. The computational complexity of learning within this model was elegantly characterized by \citet{blum1994weakly} using Fourier analysis, who introduced the concept of the SQ dimension to provide a combinatorial measure for proving lower bounds. This model has been successfully used to prove the hardness of learning for many important concept classes, including noisy linear threshold functions \citep{blum1996polynomial}, halfspaces under Gaussian marginals \citep{diakonikolas2020statistical,pmlr-v178-hsu22a}, high-dimensional Gaussians \citep{diakonikolas2017statistical}, and single-neuron networks \citep{goel2020sq}, with some problems remaining open \citep{feldman2014open}. The relevance of the SQ model has been underscored by work connecting it to other constrained learning models, such as those with memory or communication limits \citep{steinhardt2016learning}, and, most notably, to the capabilities of modern deep learning algorithms. \citet{abbe2021power} showed that learning with stochastic gradient descent (SGD) is equivalent in power to SQ learning when using large mini-batches or low-precision gradients, suggesting our hardness result has direct implications for a wide range of practical algorithms. For a comprehensive overview of the SQ model see the survey by \citet{reyzin2020survey}. Our work contributes a new, SQ hardness result for a classic concept class (semiautomata), but does so by introducing novel representation-theoretic tools to establish a high SQ dimension.

\subsection{Related DFA hardness results}
One of the simplest SQ hardness results can be established from the work of \cite{blum1994weakly} on the hardness of the parity function. In particular, the existence of a DFA with $2N+1$ states that can compute the parity of a fixed subset of any binary string of length $N$ implies that the class of $2N+1$-state DFAs is hard to learn using Statistical Queries under the uniform distribution on inputs of length $N$.

On a related note, \citet{10.1007/978-3-642-16108-7_18} demonstrate that random DNF formulas, random log-depth decision trees, and random DFAs (where both transitions and accepting states are chosen uniformly at random) cannot be weakly learned with a polynomial number of statistical queries if the distribution over the input is chosen adversarially. The paper argues that even if a structure (such as a DNF formula or a DFA) is chosen randomly, an adversary can construct a ``bad'' input distribution that makes the learning problem hard. The core technique for establishing hardness is to show that, with high probability, a hard-to-learn parity function can be embedded into the random structure. The learnability of random DFAs using Statistical Queries under the uniform input distribution was explicitly posed as an open problem by \cite{10.1007/978-3-642-16108-7_18}. Similarly, \cite{pmlr-v65-fish17a} conjectured that learning random DFAs is hard in the PAC model. Although our work considers a slightly different random semiautomata model, with transitions generated by random transpositions, we hope it can serve as a first step toward addressing these open problems.

The last negative result, which is also quite relevant to our work, is that of \citet{pmlr-v291-wang25a}. The paper establishes a computational hardness result for learning the class of $k$-fold composition functions in the SQ model. The $k$-fold composition task requires a model to compute the result of an interleaved composition of $k$ permutations provided in the input context and $k$ hidden, parametric permutations. This model can be interpreted as a restricted subclass of semiautomata with time-inhomogeneous transitions. The proof in \cite{pmlr-v291-wang25a} relies heavily on the specific structure of the hypothesis class and mostly employs recursive algebraic calculations and combinatorial facts. Our result complements that of \cite{pmlr-v291-wang25a} and is obtained by framing the problem as a random walk on the product of symmetric groups and by leveraging tools from Fourier analysis and representation theory

\section{Notation and preliminaries}
\label{sec:prelim}

Throughout the text, we use boldface to denote vectors. We reserve the notation $\vect{e_i}$ to denote the $i$-th standard basis vector. For $n \in \NN := \{1,2,\dots\}$ we use the notation $[n]$ to refer to the set $\{1,2,\dots,n\}$. For a vector $\vect{x}\in \CC^n$ we denote by $\|\vect{x}\| := \sqrt{\sum_{i=1}^n |\vect{x}_i|^2}$ its Euclidean norm. For matrices $A$ and $B$, $A \otimes B$ denotes their Kronecker product and $A \oplus B$ their direct sum. For a matrix $A$, we denote by $\|A\|_2 := \sup_{\|\vect{x}\|=1} \|A\vect{x}\|$ its spectral norm. For a finite set $S$, we denote by $\U(S)$ the uniform distribution over its elements.

\subsection{Abstract algebra}
We direct unfamiliar readers to \Cref{appendix:background}, which provides a self-contained overview of the necessary results from the representation theory of finite groups, and constitutes the foundation for our analysis. For a group $G$, we denote by $\id$ its identity element. When $G$ and $H$ are groups, $G \times H$ denotes their direct product. We denote the symmetric group on $N$ elements by $S_N$. For a permutation $g \in S_N$, we denote by $\fix(g)$ the set of fixed points of $g$. The trivial and standard representations of $S_N$ are denoted by $\triv$ and $\std$, respectively. For a representation $\rho$ and a function $f:G \to \CC$, the Fourier transform of $f$ with respect to $\rho$ is given by $\rho(f) := \sum_{g\in G}f(g)\rho(g)$. For representations $\rho$ and $\sigma$, $\rho \otimes \sigma$ denotes their tensor product.

\subsection{Semiautomata}
A semiautomaton $\mathcal{A}$ is a triplet $(\Q,\Sigma,\delta)$, where
\begin{itemize}
\item $\Q$ is a non-empty set of states.
\item $\Sigma$ is the input alphabet.
\item $\delta:\Q\times\Sigma\to \Q$ is the transition function.
\end{itemize}
We denote by $\mathcal{A}(\Sigma, \Q)$ the set of semiautomata with alphabet $\Sigma$ and state space $\Q$. Throughout this work, we will only consider semiautomata where both $\Sigma$ and $\Q$ are finite, and we denote by $N = |\Q|$ the number of states. The set of all finite words over the alphabet $\Sigma$ is denoted by $\Sigma^*$. For each symbol $a \in \Sigma$, we define the state-transition function $\delta_a: \Q \to \Q$ by $\delta_a(X) = \delta(X, a)$. Given a semiautomaton $\mathcal{A}$, we define the function $f_{\delta_{\mathcal{A}}}: \Sigma^* \times \Q \to \Q$, which assigns to each word and initial state the corresponding final state. Specifically, for a word $\w_T := a_1 \dots a_T$ and an initial state $X$, the final state is given by
$$
f_{\delta_{\mathcal{A}}}(\w_T, X) := (\delta_{a_T} \circ \delta_{a_{T-1}} \circ \dots \circ \delta_{a_1})(X).
$$

\subsection{The Statistical Query model}
\label{subsec:prelim_sq}

The Statistical Query (SQ) model, introduced by \citet{kearns1998efficient}, is a restricted learning model where the learner accesses information about a dataset through statistical queries rather than individual examples. The learner interacts with a statistical query oracle whose definition is given below.

\begin{definition}[Statistical Query oracle]
Let $\C$ be a concept class of functions mapping $\X$ to $\Y$, and let $f^*: \X \to \Y$ be a ground-truth concept in $\C$. Let $D$ be a distribution over the input space $\X$. An SQ oracle, denoted by $\operatorname{STAT}(f^*, D)$, accepts a query $(h, \tau)$, where $h: \X \times \Y \to [-1, 1]$ is a bounded statistical query function (statistic), and $\tau > 0$ is the tolerance. The oracle returns a value $v$ such that:
\[
\left| v - \E_{x \sim D}[h(x, f^*(x))] \right| \le \tau.
\]
\end{definition}

The goal is to learn $f^*$ up to arbitrarily small error after making a number of queries. 

\begin{definition}[Statistical query hardness]
Statistical query hardness for a concept class $\C$ is established by showing that there exists a ground-truth concept $f^*\in \C$ such that, in the worst case, any learner must either make a super-polynomial number of queries or use a tolerance $\tau$ such that $1/\tau$ is super-polynomial in the problem parameters to learn $f^*$.
\end{definition}

\section{The learning problem and random walk connections}
\label{sec:setting}

The concept class $\C$ consists of all functions $f_{\delta_\mathcal{A}}$ that map an initial state and an input word to the corresponding final state according to some semiautomaton $\mathcal{A} \in A(\Sigma, \Q)$, namely $\C = \left\{ f_{\delta_\mathcal{A}}: \mathcal{A} \in \mathcal{A}(\Sigma, \Q)\right\}$. With the notation of \Cref{subsec:prelim_sq}, $\mathcal{X} = \Sigma^* \times \Q$ and $\Y = \Q$. 
We investigate the hardness of learning $\C$ using statistical queries when the underlying distribution is uniform over input words of length $T$ and initial states. We will show that when $|\Sigma|$ and $T$ scale polynomially with the number of states $N$, learning $\C$ is hard.

\subsection{Modeling single semiautomata as random walks}
To prove our hardness result, we will need to construct a ``hard" set: a large family of concepts from $\C$ that are nearly uncorrelated (details are presented in \Cref{sec:sq_hardness}). We restrict our search for the hard set by considering semiautomata where the state-transition functions $\delta_a$ are permutations for each $a\in \Sigma$.\footnote{In fact, our constructions will only require each $\delta_a$ to be either the identity or a transposition.}
Under this restriction, we can model a single semiautomaton processing a uniformly random word as a random walk on the symmetric group $S_N$. Consider a semiautomaton $\mathcal{A}$ where all its state-transition functions $\delta_a$ are permutations from a set $\Gamma \subseteq \Sn$.
\begin{enumerate}
    \item \textit{The Walk State:} The state of the random walk after $t$ steps is the total permutation $P_{\w_t}$ computed so far. This is an element of the symmetric group $\Sn$.
    \item \textit{The Starting State:} Before processing any symbols, the total permutation is the identity permutation, $\id \in \Sn$, which maps every state in $\Q$ to itself. 
    \item \textit{The Transition Rule:} Let $P_{\w_t} \in \Sn$ be the permutation computed after processing the first $t$ symbols of a (random) word $\w_T = a_1 a_2 \dots a_T$, and so, $P_{\w_t} = \delta_{a_t} \circ \dots \circ \delta_{a_1}$. When the $(t+1)$-th symbol, $a_{t+1}$, is processed, the new walk state (total permutation) becomes $P_{\w_{t+1}} = \delta_{a_{t+1}} \circ P_{\w_t}$.
\end{enumerate}
Under the uniform distribution assumption on the input word, the final state $f_{\delta_{\mathcal{A}}}(\w_T, X)$ coincides with $P_{\w_T}(X)$, for any initial state $X\in \Q$.

\subsection{Modeling pairs of semiautomata as random walks}
\label{sec:two_dfa_walk}
Our hardness result relies on evaluating a correlation measure between pairs of distinct semiautomata from the hard set. This requires bounding the probability that two such semiautomata end in the same final state when started from the same uniformly chosen input word and initial state. This can also be modeled as a (coupled) random walk on the product group $G = S_N \times S_N$. Let $\mathcal{A}$ and $\mathcal{A}'$ be two semiautomata (operating on the same alphabet $\Sigma$ and state-space $\Q$) with permutation transition functions from a set $\Gamma \subseteq S_N$.
\begin{enumerate}
    \item \textit{The Walk Space:} To track both processes simultaneously, the walk state is a pair of permutations $(P_{\w_{t}}, P'_{\w_{t}})$, representing the total permutation computed by each semiautomaton after $t$ steps. This state is an element of the product group $G = \Sn \times \Sn$.
    \item \textit{The Starting State:} Both walks start at the identity, so the starting state is $(\id, \id) \in G$.
    \item \textit{The Transition Rule:} Let $(P_t, P'_t)$ be the permutation computed after processing the first $t$ symbols of a (random) word $\w_T = a_1 a_2 \dots a_T$. When the $(t+1)$-th symbol, $a_{t+1}$, is processed, the state of the joint walk transitions as follows:
    \[
    (P_{\w_{t+1}}, P'_{{\w_{t+1}}}) = (\delta_{a_{t+1}} \circ P_{\w_{t}}, \delta'_{a_{t+1}} \circ P'_{\w_{t}}).
    \]
\end{enumerate}
As before, the pair $(f_{\delta_\mathcal{A}}( \w_T, X),f_{\delta_{\mathcal{A}'}}(\w_T, X))$ coincides with  $(P_{\w_T}(X), P'_{\w_T}(X))$ under the uniform distribution assumption on the input word, for any initial state $X \in \Q$.

The single-step probability distribution of this random walk is thus defined as follows:

\begin{definition}[Single-step probability]
\label{def:single_step}
The probability distribution for a single step of the joint random walk defined in \Cref{sec:two_dfa_walk} is the function $\TDFA:G \to [0,1]$ given by:
\[
    \TDFA(g) = \frac{|\{a \in \Sigma \mid (\delta_a, \delta'_a) = g \}|}{|\Sigma|} \quad \forall\, g \in G.
\]
The support of $\TDFA$ is contained in $\Gamma \times \Gamma$.
\end{definition}

As noted above, our hardness result requires a careful analysis of the probability of agreement for two semiautomata after $T$ steps under a uniformly random input word and initial state. Namely, we are interested in the probability that for an initial state $X_0$ picked uniformly at random from $\mathcal{Q}$, after $T$ steps of the random walk, the two states $P_{\w_T}(X_0)$ and $P'_{\w_T}(X_0)$ coincide. To calculate this, we borrow tools from Fourier analysis and representation theory of the symmetric group. As it turns out, this probability is controlled solely by the Fourier transform of $\TDFA$ with respect to the irreducible representation $\Pi_0 = \std \otimes \std$ of $S_N \times S_N$. The result, along with a proof outline, is presented in \Cref{thm:p_agree}. The full proof relies on a technical lemma and is presented in \Cref{app:proof_pagree}.

\begin{restatable}[Agreement probability]{theorem}{pagree}
\label{thm:p_agree}
Consider the random walk corresponding to two semiautomata $\mathcal{A}$ and $\mathcal{A}'$ operating on the same alphabet $\Sigma$ and state-space $\Q$, as described in \Cref{sec:two_dfa_walk}. Let $N = |\mathcal{Q}|$ and let $T \in \NN$ be the input length. Let $X_0 \sim \U(\mathcal{Q})$ be a (common) initial state picked uniformly at random. Denote by $P_{\operatorname{agree}} :=P_{\operatorname{agree}}(T)$ the probability that after processing the same uniformly random word $\w_T \in \Sigma^*$, both semiautomata reach the same state. In terms of the random walk, this probability is given by 
$$P_{\operatorname{agree}}(T) = \mathbb{P}_{\w_T\sim \U(\Sigma)^{\otimes T}, X_0 \sim \U(\mathcal{Q)}}\left(P_{\w_T}(X_0) = P'_{\w_T}(X_0)\right).$$
Then
$$P_{\operatorname{agree}} = \frac{1}{N} + \frac{1}{N} \vect{v}^\top M_{\Pi_0}^T \vect{v},$$
where $M_{\Pi_0}$ is the Fourier transform of the single-step distribution $\TDFA$ with respect to the irreducible representation $\Pi_0 = \std \otimes \std$ of $G$ and $\vect{v} = \sum_{i=1}^{N-1} \vect{e_i} \otimes \vect{e_i}$.
\end{restatable}

\begin{proof}[Proof outline]
    The proof begins by expressing the agreement probability, $P_{\operatorname{agree}}$, as an expected value over the final states of the random walk on the product group $S_N \times S_N$. This is written as 
    \begin{equation}
        \label{eq:p_agree_outline}
        P_{\operatorname{agree}} = \sum_{(g,h) \in G} P_T(g, h) \cdot \frac{|\fix(h^{-1}g)|}{N},
    \end{equation}
    where $P_T$ denotes the distribution of the walk after $T$ steps. The distribution $P_T$ is then decomposed as a sum of the uniform distribution on $S_N \times S_N$ and an error term: $P_T(g,h) = P_U(g,h) + \operatorname{err}(g, h)$. Substituting this into \Cref{eq:p_agree_outline}, we find that the uniform part of the sum evaluates cleanly to the baseline agreement probability of $1/N$. The remaining error term is then analyzed using the Plancherel formula from Fourier analysis (\Cref{lemm:planch}), and is given by the following sum over non-trivial irreducible representations:  
    \[
    \frac{1}{N}\cdot \frac{1}{|G|} \sum_{\Pi \neq \triv} d_\Pi \Tr\left( M_\Pi^T \Pi(f)^* \right),
    \]
    where $M_{\Pi} = \Pi(P_T)$ and $f$ is given by $f(g,h) = |\fix(h^{-1}g)|$. The key insight is that the Fourier transform of $f$ is non-zero only for a single irreducible representation, $\Pi_0 = \std \otimes \std$ (\Cref{lemm:fourier_fix}). This collapses the sum above into a single term, yielding the final expression.
\end{proof}

\Cref{thm:p_agree} is central to our analysis: it reduces the task of bounding $P_{\operatorname{agree}}$ for our constructions from having to account for all (exponentially many) irreducible representations of $S_N \times S_N$,\footnote{The group $S_N \times S_N$ has $p(N)^2$ irreducible representations, where $p(N)$ is the partition function (c.f. \Cref{rem:num_irreps}). The partition function grows exponentially \citep{hardy1918partition} and hence there are exponentially many irreducible representations.} to considering just a single one. In the next section, we introduce a randomized construction used to generate a family of semiautomata and determine the required alphabet size $|\Sigma|$ and word length $T$ that ensure $|P_{\operatorname{agree}} - 1/N| \leq 1/N!$, a bound necessary for establishing SQ-learning hardness.
\section{Randomized construction of the hard set}
\label{sec:randomized_family}

In this section, we present a randomized construction of a family of semiautomata that will form the basis of our hard set, which will be used to derive the statistical query lower bound. In what follows, we assume that all semiautomata operate on the same, fixed set of $N$ states. The construction depends on a parameter $k$ that controls the size of the alphabet. We construct $M$ semiautomata operating on a common alphabet consisting of $k$ labeled copies of every transposition in $S_N$. For each semiautomaton $i$ and symbol $\tau$, we flip an independent fair coin to decide whether the state-transition function corresponding to $\tau$ acts as the identity or $\tau$ itself. Thus, each semiautomaton is defined by a random mask of swap and identity symbols. The formal definition is given below:

\begin{definition}[The randomized $(k, M)$-shuffle family construction] \label{def:dfa_family} We construct a (random) family of semiautomata operating on the same alphabet with the following properties:
\begin{enumerate}
    \item Let $\Sigma_1 = \{\tau_1,\dots,\tau_{\binom{N}{2}}\}$ be an enumeration of all transpositions in $S_N$. The alphabet $\Sigma_k$ is the disjoint union of $k$ copies of $\Sigma_1$. The size is $|\Sigma_k| = k \binom{N}{2}$.
    \item A family $\mathcal{F} = \{\delta_1, \dots, \delta_M\}$ is generated by assigning to each symbol (transposition) $\tau \in \Sigma_k$ and each semiautomaton $\delta_i$ an independent random variable $b_i(\tau) \in \{0, 1\}$ with $\Prob(b_i(\tau)=1)=1/2$. Each transition function $\delta_i : \Q \times \Sigma_k \to \Q$ is given by
    \[
    \delta_i(X, \tau) = \begin{cases} \tau(X) & \text{if } b_i(\tau) = 1 \\ X & \text{if } b_i(\tau) = 0 \end{cases} \quad \text{for all } X \in \Q \text{ and } \tau \in \Sigma_k.
    \]
    For simplicity, we overload our notation to identify semiautomata by their transition functions.
\end{enumerate}
We refer to the resulting family $\F$ as a randomized $(k,M)$-shuffle family.
\end{definition}

Note that while \Cref{def:dfa_family} forces the alphabet to consist of transpositions, this is done only to discharge notation and enhance the readability of our proofs. Indeed, it is easy to see that the construction can work with any alphabet, so long as its size is $k\binom{N}{2}$, simply by appropriately remapping the alphabet characters to transpositions. 

The main goal of this section is to find appropriate values for the alphabet parameter $k$ and the input word length $T$ such that, with high probability, any two semiautomata from a randomized $(k, N!)$-shuffle family $\F$ satisfy $P_{\operatorname{agree}} \leq 1/N + 1/N!$. Given the form of $P_{\operatorname{agree}}$ derived in \Cref{thm:p_agree}, the first step is to show that, with high probability, the spectral norm of $M_{\Pi_0}$ is bounded away from $1$ for any choice of distinct semiautomata from $\F$. This is done in \Cref{thm:spectral_bound_random}, which we state and provide a proof outline here. The full proof is deferred to \Cref{app:proofs_spectral}. 

\begin{restatable}{lemma}{spectralbound}
\label{thm:spectral_bound_random}
Let $N \geq 4$ and consider a randomized $(k, M)$-shuffle family $\F$ with $M = N!$ and 
$$k \geq \frac{16(3N + 1)}{3(N - 1)} \left[ N \ln N + \ln \binom{N!}{2} + 2\ln(N - 1) \right].$$ Then any two distinct semiautomata $\delta_i, \delta_j \in \mathcal{F}$ satisfy $\norm{M_{\Pi_0}^{i,j}} \le  1 - \frac{1}{2N}$ with probability at least $1 - \exp(-N\ln N)$ where $\Pi_0 = \std \otimes \std$ and $M_{\Pi_0}^{i,j}$ denotes the Fourier transform of the single-step probability distribution for the joint walk according to $\delta_i$ and $\delta_j$.
\end{restatable}

\begin{proof}[Proof outline]
    The proof follows a two-step ``expectation-concentration" argument. First, we calculate the expectation of the Fourier transform matrix, $\E[M_{\Pi_0}^{i,j}]$, by averaging over all random choices in our semiautomaton construction. The canonical way by which randomness is introduced in the construction (through the fair coin toss), along with the use of transpositions as state-transition functions, ensures that the expectation has a manageable form. Namely, by applying a generalization of Schur’s Lemma (\Cref{corr:schur_gen}), we show that the expected matrix is similar to a direct sum of scaled identity matrices, and that its spectral norm is $1 - \frac{1}{N-1}$. Second, we show that the matrix for a randomized $(k, N!)$-shuffle family concentrates around this expectation. We use the Matrix Bernstein inequality     (\Cref{lem:bernstein}) to prove that the deviation $\|M_{\Pi_0}^{i,j} - \E[M_{\Pi_0}^{i,j}]\|_2$  is very small with high probability, provided the alphabet size parameter $k$ is sufficiently large. By combining the bound on the expectation's norm with the high-probability bound on the deviation via the triangle inequality, we arrive at the final desired bound of $1-\frac{1}{2N}$.
\end{proof}

Combining \Cref{thm:spectral_bound_random} with \Cref{thm:p_agree}, we can directly derive an upper bound on $P_{\operatorname{agree}}$ that converges to $1/N$ exponentially in $T$.

\begin{restatable}{lemma}{probt}\label{lem:prob_of_agreement}
For $M=N!$ and $N\ge 4$, if we choose the alphabet parameter 
$$k \geq \frac{16(3N + 1)}{3(N - 1)} \left[ N \ln N + \ln \binom{N!}{2} + 2\ln(N - 1) \right],$$ then any pair of semiautomata $(\delta_i, \delta_j)$ from a randomized $(k, M)$-shuffle family $\mathcal{F}$ satisfies
\[\left| P_{\operatorname{agree}} - \frac{1}{N}\right| \leq \left(1 - \frac{1}{2N}\right)^{T} \]
with probability at least $1 - \exp(-N\ln N)$.
\end{restatable}

Lastly, we show how to choose $T$ to achieve the desired bound.

\begin{restatable}{theorem}{randt}
\label{thm:randomized_t}
For $T \geq 2N\ln(N!)$ and $k$ and $M$ as given in \Cref{lem:prob_of_agreement}, we have that any pair of distinct semiautomata $(\delta_i, \delta_j)$ from a randomized $(k,M)$-shuffle family $\mathcal{F}$ satisfies $|P_{\operatorname{agree}} - 1/N| \leq 1/N!$ with probability at least $1 - \exp(-N\ln N)$.
\end{restatable}

The proofs of \Cref{lem:prob_of_agreement} and \Cref{thm:randomized_t} are deferred to \Cref{app:proofs_spectral}. Notice that to achieve the desired bound, both $k$ and $T$ can be chosen to be polynomial in $N$. This is expressed in the following remark:

\begin{remark}
\label{rem:randomized_orders}
    Since $\ln(N!) = \mathcal{O}(N\ln N)$, the high probability result of \Cref{thm:randomized_t} holds for a choice of $\,T= \mathcal{O}(N^2 \ln N)$ and $k=\mathcal{O}(N\ln N)$. In that case, the total alphabet size is $\mathcal{O}(N^3\ln N)$.
\end{remark}

Note that choosing $b_i(\tau)$ as a $\operatorname{Ber}(1/2)$ random variable is optimal. Intuitively, biasing the transitions toward either the identity or transpositions makes pairs of semiautomata easier to distinguish. As a result, a larger alphabet and longer inputs are needed to achieve indistinguishability. This is formalized in the following remark:

\begin{remark}
     In the randomized construction of \Cref{def:dfa_family}, defining each $b_i(\tau)$ as an independent $\operatorname{Ber}(p)$ random variable with $p \in (0,1)$ requires an alphabet of size at least $\mathcal{O}\left(\frac{N^3 \ln N}{p^2 (1-p)^2}\right)$ and an input word length of at least $\mathcal{O}\left(\frac{N^2 \ln N}{p(1-p)}\right)$. Both quantities are minimized when $p = 1/2$, showing that a fair coin flip is the most efficient choice for establishing hardness.
\end{remark}

As a result of independent interest, in \Cref{thm:rand_lb} we prove that the choice of $T$ above is asymptotically optimal. This essentially shows that to achieve the desired bound on $P_{\operatorname{agree}}$ with high probability, we must choose $T = \Omega(N^2 \ln N)$. In terms of the random walk, this shows that the best-case mixing time of the walk corresponding to the randomized $(k, M)$-shuffle family with $k$ and $M$ chosen as in \Cref{lem:prob_of_agreement} is $\Omega(N^2 \ln N)$. The proof of \Cref{thm:rand_lb} follows easily by the analysis carried out in the proof of \Cref{thm:spectral_bound_random} and is deferred to \Cref{app:proof_rand_lb}.

\begin{restatable}[Tightness of mixing time]{theorem}{randlb}
\label{thm:rand_lb}
    Let $N\geq 5$, and let $k$ and $M$ be as in \Cref{lem:prob_of_agreement}. For any pair of distinct semiautomata $(\delta_i, \delta_j)$ from a randomized $(k, M)$-shuffle family to satisfy  $|P_{\operatorname{agree}} - 1/N| \leq 1/N!$ with probability $1-\exp(-N\ln N)$, the input word length must be at least $T = \Omega(N^2 \ln N)$.
\end{restatable}

\section{Statistical Query hardness}
\label{sec:sq_hardness}

This section is organized into two parts, with the goal of establishing the hardness of learning semiautomata in the Statistical Query framework. In \Cref{sub:sq_results}, we introduce the necessary notions of pairwise correlation and SQ dimension, together with a lower bound theorem for SQ learning. These results are stated in a general setting for concept classes whose functions take finitely many values, thereby extending earlier results from the binary case. In \Cref{sub:dfa_hardness}, we apply those results to obtain our main hardness result for semiautomata.

\subsection{General results}
\label{sub:sq_results}
Throughout this section, we assume a general concept class $\C$ consisting of functions $f: \X \to \Y$, where $\X$ and $\Y$ are abstract spaces with $|\Y| < \infty$. Likewise, we assume an arbitrary distribution $D$ over $\X$. The definitions and results we present are straightforward extensions of the well-known binary case $\Y = \{0,1\}$.

\begin{definition}[Pairwise correlation]
Let $\F$ be a family of functions $f: \X \to \Y$. Let $D$ be a distribution over $\X$. When $\Y$ is a finite set, the correlation between two distinct functions $f_i, f_j \in \F$ under $D$ is defined in terms of the agreement probability as:
\[
\chi(f_i, f_j) = \Prob_{x \sim D}[f_i(x) = f_j(x)] - \frac{1}{|\Y|}.
\]
\end{definition}

SQ lower bounds are typically proven by constructing a large family of functions that are nearly uncorrelated, formalized by the concept of SQ Dimension.

\begin{definition}[SQ dimension]
Let $\C$ be a concept class over a distribution $D$. The statistical dimension $\SQDim_\C^D$ is the largest integer $d$ such that there exists a subset $\{f_1, \dots, f_d\} \subseteq \C$ of nearly uncorrelated functions satisfying $|\chi(f_i, f_j)| \leq 1/d$ for all $i, j \in [d]$ with $i\neq j$.
\end{definition}

The following theorem, whose proof is given in \Cref{app:proof_sqlb}, establishes a lower bound on the number of queries a learner must make in the worst case in terms of the SQ dimension:

\begin{restatable}[SQ lower bound]{theorem}{sqlb}
\label{thm:sq_lower_bound}
Let $\C$ be a concept class and suppose $\SQDim_\C^D \geq d$. Then any SQ learner using tolerance $\tau >0$ requires, in the worst case, at least 
$$q \geq \frac{(d-1)(d\tau^2 -|\Y|)}{2d(|\Y|-1)}$$ 
queries to learn the ground-truth concept $f^*$.
\end{restatable}

\subsection{The case of semiautomata}
\label{sub:dfa_hardness}
We are now prepared to establish our main hardness result for the class of semiautomata. Recall that the concept class is defined as $\C = \left\{ f_{\delta_\mathcal{A}}: \mathcal{A} \in \mathcal{A}(\Sigma, \Q)\right\}$ and that the underlying distribution $D_T$ on $\Sigma^* \times \Q$ is uniform over all states in $\Q$ and all words of length $T$ in $\Sigma^*$. Denoting by $N = |\Q|$ the number of states, our hardness result can be stated as follows:

\begin{theorem}
\label{thm:dfa_hard}
Suppose the alphabet size and word length satisfy $|\Sigma| = \Omega(N^3 \ln N)$ and $T = \Omega(N^2 \ln N)$, respectively. Then for all $c > 0$, any SQ learner for $\C$ under the distribution $D_T$ must, in the worst case, either make $\omega(N^c)$ queries or use a tolerance of $o(N^{-c})$. Hence, learning $\C$ with statistical queries is hard.
\end{theorem}

\begin{proof}
By \Cref{rem:randomized_orders}, a randomized $(k, N!)$-shuffle family $\F = \{\delta_1,\dots,\delta_{N!}\}$ with $k=\mathcal{O}(N\ln N)$ satisfies $|P_{\operatorname{agree}} - 1/N| \leq 1/N!$ for any two distinct semiautomata. In terms of pairwise correlation, this translates to $|\chi(\delta_i,\delta_j)| \leq 1/N!$ for all $i \neq j$. We have thus shown that the SQ dimension of $\C$ is at least $N!$. The result follows from  \Cref{thm:sq_lower_bound}.  
\end{proof}
\section{Conclusion and future work}
\label{sec:conclusion}

We have shown that learning semiautomata with $N$ states and alphabet size $\Omega(N^3\ln N)$ is computationally hard in the SQ model under the uniform distribution over initial states and inputs of length $\Omega(N^2\ln N)$. Using a novel link between semiautomata distinguishability and mixing properties of random walks on $S_N\times S_N$, we constructed a family of $N!$ nearly uncorrelated semiautomata, establishing SQ hardness. In terms of future work, we identify the following interesting research directions: 
\begin{enumerate}
    \item \textbf{Studying hardness for different parameter regimes:} Our hardness results are established for semiautomata with alphabet size $\Omega(N^3 \ln N)$ and input length $\Omega(N^2 \ln N)$. An important open question is whether comparable hardness results hold in other parameter regimes, or whether more favorable positive results can be obtained. Addressing this may require considering semiautomata whose transition functions extend beyond transpositions, or even beyond permutations altogether. A different line of analysis might also reveal an interesting tradeoff: smaller alphabet sizes could necessitate longer input lengths to demonstrate hardness. Conversely, when both the alphabet size and input length are sufficiently restricted, the class may even turn out to be efficiently learnable.
    \item \textbf{Tightness of SQ-hardness theorem:} \Cref{thm:dfa_hard} shows that any SQ learner must either make a super-polynomial number of queries or operate with super-polynomially small tolerance. A natural question is whether there exist algorithms that realize these extremes individually: one incurring the query cost and another incurring the tolerance cost.
    \item \textbf{Learnability in neural architectures:} Our results indicate that conventional gradient-based training of expressive models such as Transformers may not suffice to learn the general class of semiautomata from a uniform data distribution, even though Transformers are known to be capable of efficiently simulating semiautomata \citep{liu2023transformers}. This motivates an investigation into alternative training paradigms that could circumvent this limitation, like curriculum learning or reinforcement learning.
\end{enumerate}

\newpage
\section*{Acknowledgments}

K.~Fountoulakis would like to acknowledge the support of the Natural Sciences and Engineering Research Council of Canada (NSERC). Cette recherche a \'et\'e financ\'ee par le Conseil de recherches en sciences naturelles et en g\'enie du Canada (CRSNG), [RGPIN-2019-04067, DGECR-2019-00147].

G.~Giapitzakis would like to acknowledge the support of the Onassis Foundation - Scholarship ID: F ZU 020-1/2024-2025.

In preparing this paper, we made use of Google’s Gemini 2.5 Pro. The randomized construction in \cref{sec:randomized_family}, which is used in \cref{thm:randomized_t}, was initially suggested by the model, prompting us to draw a connection with random walks on groups. This connection naturally led to the introduction of tools from group and representation theory that form the theoretical backbone of our results.

\bibliographystyle{plainnat}
\bibliography{references}

\newpage

\appendix

\section{More on related work}
\label{app:add_lit}

In this section, we review a broad range of DFA learnability results outside the statistical query model. These include results in the PAC model, learning with margin-based guarantees, and query-based learning. We cover both positive and negative results.

\subsection{Learnability within the PAC model}

\textbf{Positive results.} One of the first positive results was given in \citep{clark2004pac,PALMER200718}, where, with structural assumptions, the PAC-learnability of probabilistic DFAs is proved. Another approach restricts the distribution of examples; for instance, \citet{parekh2001learning} showed that ``simple" DFAs are PAC-learnable under the universal distribution. Early work by \citet{ron1995learnability} showed that acyclic Probabilistic Deterministic Finite Automata (PDFAs) were learnable. This was followed by a series of results on learning general PDFAs using state-merging techniques \citep{carrasco1994learning, carrasco1999learning, lang1998results}. A key theoretical breakthrough came by \citet{clark2004pac}, which proved the PAC-learnability of PDFAs by parameterizing the complexity by the ``distinguishability" of the automaton's states. This line of research culminated in the work of \citet{balle2013learning} (and the corresponding thesis \citep{balle2013learningfinite}), which rigorously characterized state-merging algorithms. They proved that PDFAs are efficiently learnable, but crucially, that the complexity must depend polynomially on the inverse of the distinguishability parameter. Finally, the work by \citet{10.1145/225298.225342} investigates the PAC-learnability of deterministic finite automata (DFAs), focusing on the subclass of bounded-width branching programs. It shows that width-2 branching programs can be efficiently learned (both distribution-free and properly under the uniform distribution), while learning width-3 branching programs is at least as hard as learning disjunctive normal forms.

\textbf{Negative results.} The first hardness results for DFA were established in the works of \citet{ANGLUIN1978337,GOLD1978302,10.1145/138027.138042} under standard complexity assumptions. Subsequent works by \citet{10.1145/174644.174647} and \cite{kearns1994introduction} established representation-independent hardness results for finite automata within the PAC model by utilizing ``bad" distributions over the input so that any efficient learning algorithm that achieves a polynomial advantage over random guessing will break various cryptographic hardness assumptions. We remark that establishing a hardness result for the PAC-learnability of random DFAs, i.e., DFAs with transition functions and accepting states chosen uniformly at random, remains an open problem \citep{pmlr-v65-fish17a}.

\subsection{Learnability with margin-based guarantees}

A different line of work from the state-merging approach considers margin-based guarantees, a framework weaker than PAC learning \citep{10.1007/11894841_24,10.1007/978-3-540-72927-3_26,KONTOROVICH2008223,10.5555/1577069.1577108}. This approach embeds piecewise-testable languages and other discrete concepts into high-dimensional spaces using kernel methods and identifies languages via hyperplanes.

\subsection{Learnability based on queries}

The seminal $\textup{L}^*$ algorithm by \citet{angluin1987learning} showed that DFAs are efficiently learnable in polynomial time using membership and equivalence queries. This foundational result established a fundamental dichotomy: DFAs are easy to learn with an interactive teacher but hard to learn from random examples alone. The power of queries has been extensively studied by \citet{angluin1988queries}, with research exploring the query complexity of learning \citep{gavalda1997query} and establishing lower bounds on the number of queries required \citep{pmlr-v217-kruger23a}. This line of work has been extended to more complex automata models, such as register automata \citep{bollig2009fresh}, symbolic weighted automata \citep{suzuki2021query}, and nominal DFAs \citep{zhou2024querylearningadvicenominal}, and has been surveyed in works like \citep{isberner2015foundations}. More recently, researchers have explored using modern tools like Recurrent Neural Networks \citep{weiss2018extracting} and Large Language Models \citep{vazquez-chanlatte2024l} to act as oracles. Other works have demonstrated learnability with different powerful oracles, such as those that provide confidence scores \citep{wu2023learning}. Another approach has been to restrict the class of learnable automata, leading to positive SQ results for specific subclasses such as shuffle ideals \citep{angluin2013learning}. 
\section{Abstract algebra background}\label{appendix:background}
This section provides a concise overview of the essential concepts from group theory and representation theory that are used in our proofs. It is divided into two parts: in \Cref{sub:basics} we present foundational definitions and results from both theories, while in \Cref{sub:s_n} we focus specifically on the representation theory of the symmetric group $S_N$, which plays a central role in our analysis.

\subsection{Group and representation theory basics}
\label{sub:basics}
To enhance readability, this section is further divided into two subsections: in \Cref{subsub:gt} we present an overview of elementary definitions and facts from group theory, whereas in \Cref{subsub:rep} we provide a concise summary of key definitions and results from the representation theory of groups. The results presented in \Cref{subsub:gt} are standard and can be found in any introductory text on abstract algebra, such as \cite{lang2005algebra}. The results in \Cref{subsub:rep} are likewise standard within the representation theory literature and can be found, for instance, in \cite{serre1996linear}.
 
\subsubsection{Group theory overview}
\label{subsub:gt} 
We begin by defining the notion of a group:

\begin{definition}[Group]
    A \emph{group} is a non-empty set $G$ endowed with a binary operation $\star: G \times G \to G$ such that the following three axioms are satisfied:
    \begin{enumerate}
        \item \textbf{Associativity:} For all $a,b,c \in G$ one has $(a\star b)\star c = a \star (b \star c)$.
        \item \label{item:ident} \textbf{Identity element:} There exists an element $\id  \in G$ such that for all elements $g \in G$ one has $g \star \id = \id \star g = g$.
        \item \label{item:inv} \textbf{Inverse element:} For each element $g \in G$ there exists an element $h \in G$ such that $g \star h = h \star g = \id$.
    \end{enumerate}
\end{definition}

From the above definition, it follows that $\id$ is unique and that for each $g \in G$, the element $h \in G$ that satisfies \Cref{item:inv} is unique. Hence, we may write $g^{-1}$ to refer to that element. Furthermore, when it is clear from the context, for $g,h \in G$ we will write $gh$ instead of $g \star h$ and also use the notation $g^n := g \star \dots \star g$. The \emph{order} of $G$, denoted by $|G|$, is the size of the underlying set $G$. The \emph{order of an element} $g\in G$, denoted by $|g|$, is the smallest integer $n \in \NN$ such that $g^n = \id$. If no such $n$ exists, we say that $g$ has infinite order. The order of every element $g \in G$ divides the order of the group $G$ (Lagrange's theorem, Proposition 2.2 in \cite{lang2005algebra}), and in particular $|g| \leq |G|$. The latter shows that when $G$ is finite, every element has finite order and that $g^{|G|} = \id$. We can now define the notion of a homomorphism, which is a mapping between groups that respects their structure. 

\begin{definition}
    Let $(G, \star)$ and $(H, *)$ be groups. A \emph{homomorphism} is a mapping $f: G \to H$ such that for all $g_1, g_2 \in G$ one has 
    $$f(g_1 \star g_2) = f(g_1) * f(g_2).$$
\end{definition}

From the definition, it follows that $f$ maps the identity of $G$ to the identity of $H$, namely $f(\id_G) =\id_H$.\footnote{To see why, write $f(\id_G) =f(\id_G \id_G) = f(\id_G)^2$ and pre-multiply by $f(\id_G)^{-1}$.} When $f$ is bijective, we say that it is an \emph{isomorphism} and the groups $G$ and $H$ are called isomorphic. In that case we write $G \simeq H$. From an algebraic perspective, isomorphic groups are essentially the same group. An isomorphism $f: G \to G$ is called an \emph{automorphism} of $G$.

\begin{example}
\label{ex:grp} Here we present two key examples, which we will encounter frequently in this work:
    \begin{enumerate}[i)]
        \item The set of all permutations of ${1,2,\dots,N}$, equipped with the operation of function composition, where applying $\sigma \circ \pi$ means first applying $\pi$ followed by $\sigma$, forms a group, denoted by $S_N$. Its order is $|S_N| = N!$.
        \item Let $V$ be a vector space over a field $\mathbb{F}$. The set of automorphisms of $V$, i.e., the set of all bijective linear transformations $V\to V$, endowed with the function composition operation, is a group, denoted by $\GL(V)$. When $V$ is a finite-dimensional vector space of dimension $d < \infty$, the group $\GL(V)$ is isomorphic to the group $\GL(d, \mathbb{F})$ of invertible $d \times d$ matrices over $\mathbb{F}$ with the operation of matrix multiplication. The isomorphism maps each automorphism of $V$ to its matrix representation, allowing us to view elements of $\GL(V)$ as $d \times d$ invertible matrices with entries in $\mathbb{F}$.
    \end{enumerate}
\end{example}

Next, we define the direct product of two groups as follows:

\begin{definition}[Direct product]
    Let $(G, \star)$ and $(H, *)$ be groups. We endow the Cartesian product $G \times H$ with the binary operation $\circ :(G \times H) \times (G \times H) \to G\times H$ defined component-wise:
    $$(g_1, h_1) \circ (g_2, h_2) = (g_1 \star g_2, h_1 * h_2).
    $$
    The resulting structure $(G \times H, \circ)$ satisfies the group axioms and is called the \emph{direct product} of $G$ and $H$, denoted by $G \times H$.
\end{definition}

Lastly, we conclude this section by discussing conjugacy, a key property when studying group representations.

\begin{definition}[Conjugacy]
\label{def:conj}
    Let $G$ be a group and let $g, h \in G$. The elements $g$ and $h$ are called \emph{conjugate} if there exists $a \in G$  such that $h = aga^{-1}$. This is an equivalence relation whose equivalence classes are called \emph{conjugacy classes}.
\end{definition}

\subsubsection{Representation theory overview}
\label{subsub:rep}
At a high level, representation theory provides a powerful framework for studying groups by translating their abstract algebraic structure into the concrete language of linear algebra. We begin by presenting the definition of a representation:

\begin{definition}[Representation]
    Let $G$ be a group. A \emph{representation} of $G$ over a vector space $V$ over some field $\mathbb{F}$ is a homomorphism $\rho : G \to \GL(V)$. The dimension of the representation, denoted by $d_\rho$, is the dimension of the vector space $V$.
\end{definition}

From this point forward, we will assume that all groups are finite, all representations are finite-dimensional, and that $\mathbb{F}=\CC$. These assumptions hold in all cases relevant to our work. By the discussion in \Cref{ex:grp}, we can therefore view the image of each $g \in G$ under $\rho$, i.e. $\rho(g)$, as an invertible $d_\rho \times d_\rho$ complex matrix. From the definition, we immediately get that $\rho(\id) = I_{d_\rho}$ and that $\rho(g^{-1}) = \rho(g)^{-1}$. Furthermore, it can be shown that $\rho(g)$ can always be chosen to be \emph{unitary} (c.f. Section 1.3 in \cite{serre1996linear}), in which case $\rho(g^{-1}) = \rho(g)^*$.

\begin{example}
\label{ex:triv}
    The trivial representation, denoted by $\triv$, maps each group element to the $1 \times 1$ identity matrix. Hence, $d_{\triv} =1$.
\end{example}

We can define the notion of an isomorphism between two representations as follows: 

\begin{definition}[Isomorphic representations]
\label{def:isom_rep}
Two representations $(\rho, V), (\sigma, V')$ of a group $G$ are called isomorphic if there exists a linear isomorphism $\tau: V \to V'$ such that $\tau \circ \rho(g) = \rho'(g)\circ \tau$ for all $g\in G$. In that case, we write $\rho \simeq \sigma$.
\end{definition}
    
Next, we define the concept of irreducible representations. Maschke's theorem
(Theorem 2 in \cite{serre1996linear}) guarantees that every representation is isomorphic to a direct sum of irreducible representations, allowing us to restrict our study to only those representations that are irreducible.

\begin{definition}[Irreducible representation]
    A representation $\rho: G \to \GL(V)$ over a vector space $V$ is called \emph{irreducible} if the only subspaces of $V$ that are left invariant by every $\rho(g)$ are $\{0\}$ and $V$. In symbols, $\rho$ is irreducible if the following holds for any subspace $W \subseteq V$:
    $$\rho(g)(W) \subseteq W \;\; \text {for all } g \in G \implies W = \{0\} \text{ or } W = V$$
    We use the abbreviation \emph{irreps} to refer to the irreducible representations. The set of all irreducible unitary representations of $G$ is denoted by $\hat{G}$.
\end{definition}

Next, we define the direct sum of two representations.

\begin{definition}[Direct sum of representations]
\label{def:direct_sum}
    Let $\rho_1: G \to \GL(V_1)$ and $\rho_2: G \to \GL(V_2)$ be two representations of $G$. Their direct sum is defined as the representation $\rho_1 \oplus \rho_2: G \to \GL(V_1 \oplus V_2)$ and its action is given by
    $$(\rho_1 \oplus \rho_2)(g) = \begin{bmatrix}
        \rho_1(g) & 0 \\ 0 & \rho_2(g)
    \end{bmatrix}$$
    for all $g\in G$. The definition trivially extends to a finite number of representations.
\end{definition}

In what follows, we characterize the irreducible representations of the direct product via the irreducible representations of the individual groups. To do so, we need to define the tensor product representation:

\begin{definition}[Tensor product of representations]
    Let $G$ and $H$ be two groups and let $\rho: G \to \GL(V_1)$ and $\sigma: H\to \GL(V_2)$ be representations of $G$ and $H$, respectively. We define the \emph{tensor product representation} $\rho \otimes \sigma$ of the direct product $G\times H$ over $V_1 \otimes V_2$ as the representation
\begin{align*}
    \rho \otimes \sigma : G \times H &\to \GL(V_1\otimes V_2) \\
    (g, h) &\mapsto \rho(g) \otimes \sigma(h)
\end{align*}
The dimension of $\rho \otimes \sigma$ is $d_{\rho\otimes \sigma} = \dim(V_1\otimes V_2) = \dim(V_1)\dim(V_2) = d_\rho d_\sigma$.
\end{definition}

 For finite groups, the action of the tensor product representation $\rho \otimes \sigma$ on an element $(g,h) \in G\times H$ can be realized as the Kronecker product of the corresponding matrices $\rho(g)$ and $\sigma(h)$. As such, standard Kronecker product properties apply.

The following theorem yields a complete characterization of the irreducible representations of $G \times H$.

\begin{theorem}[Theorem 10 in \cite{serre1996linear}]
\label{thm:tensor_prod}
    Every irreducible (unitary) representation of $G \times H$ is isomorphic to a tensor product of irreducible (unitary) representations of $G$ and $H$. Symbolically, if $\Pi = G \times H$ we have
    $$\hat{\Pi} \simeq \left\{\rho \otimes \sigma: \rho\in \hat{G}, \sigma \in \hat{H}\right\}.$$
\end{theorem}

We now move to the study of characters, an essential tool in representation theory that captures key information about a representation. 

\begin{definition}[Character]
    Let $\rho: G \to \GL(V)$ be a representation of a group $G$ over a vector space $V$. Define the \emph{character} of $\rho$ as the mapping $\chi_\rho : G \to \CC$ given by $\chi_\rho(g) = \Tr[\rho(g)]$ for each $g \in G$.
\end{definition}

\begin{definition}[Class function]
    A function $f: G \to \CC$ that is constant on every conjugacy class of $G$ is called a \emph{class function}.
\end{definition}

From the cyclic property of the trace, we immediately get the following:

\begin{remark}
    The character of any representation is a class function.
\end{remark}

 By \Cref{def:isom_rep} we can also deduce the following:

\begin{remark}
    The characters of two isomorphic representations of $G$ are equal. Namely, if $\rho \simeq \sigma$ then $\chi_{\rho}(g) = \chi_{\sigma}(g)$ for all $g \in G$.
\end{remark}
 \begin{proof}
     Since $\rho \simeq \sigma$, there exists an invertible matrix $P$ such that $\sigma(g) = P\rho(g)P^{-1}$ for all $g\in G$. Hence,
     $$\chi_{\sigma}(g)=\Tr(\sigma(g)) = \Tr(P\rho(g)P^{-1})=\Tr(P^{-1}P\rho(g))=\Tr(\rho(g)) = \chi_{\rho}(g).$$
 \end{proof}
For two complex-valued functions $f:G\to \CC$ and $h:G \to \CC$ we define a scalar product as

\begin{equation}
\label{eq:inner_prod}
\langle f, h\rangle = \frac{1}{|G|}\sum_{g\in G} f(g) \overline{h(g)}.
\end{equation}

It can be shown that \Cref{eq:inner_prod} defines an inner product when restricted to the set of class functions. In particular, for the case of characters, we have the following theorem:

\begin{theorem}[Orthogonality relations, Theorem 3 in \cite{serre1996linear}]
\label{thm:orhtogonality}
If $\chi$ is the character of an irreducible representation of $G$, then $\langle \chi, \chi\rangle = 1$. Furthermore, if $\chi, \chi'$ are the characters of two non-isomorphic irreducible representations of $G$, then $\langle\chi, \chi'\rangle = 0$.
\end{theorem}

Next, we present Schur's orthogonality relations, also known as the ``Great Orthogonality Theorem", which provides a powerful tool for simplifying several calculations in our proofs:

\begin{theorem}[Great Orthogonality Theorem, Corollaries 1-3 in \cite{serre1996linear}]
\label{thm:GOT}
   Let $\rho$ and $\sigma$ be two non-isomorphic irreducible unitary representations $\rho$ and $\sigma$ of $G$. Let $\rho(g)_{ij}$ and $\sigma(g)_{kl}$ denote matrix elements of $\rho(g)$ and $\sigma(g)$, respectively. Then
    $$\sum_{g\in G} \rho(g)^*_{ij} \sigma(g)_{kl} = 0.$$ 
    For matrices of the same irreducible unitary representation $\rho$, the relation is:
    $$\sum_{g\in G} \rho(g)^*_{ij} \rho(g)_{kl} = \frac{|G|}{d_\rho} \delta_{il} \delta_{jk},$$
    where $\delta$ denotes the Kronecker delta and $A^*$ the conjugate transpose of a matrix $A$.
\end{theorem}

The strength of representation theory lies in its ability to provide a clean analogue of Fourier analysis when working on groups. This perspective proves especially valuable when analyzing random processes on groups. We provide the definition below:

\begin{definition}[Fourier transform]
\label{def:fourier}
    Let $G$ be a group and let $\rho: G \to \GL(V)$ be a representation. The \emph{Fourier transform} is a matrix valued function acting on $G$ given by 
    $$\rho(f) = \sum_{g\in G} f(g)\rho(g).$$
\end{definition}

The convolution of two functions acting on the same group is defined as follows:

\begin{definition}[Convolution]
    Let $G$ be a group and let $f_1: G \to \CC$ and $f_2: G \to \CC$ be functions acting on $G$. The convolution $f_1 * f_2 : G \to \CC$ is given by 
    $$(f_1 * f_2)(g) = \sum_{h \in G} f_2(gh^{-1})f_1(h).$$
\end{definition}

The following textbook result gives the Fourier transform of the convolution of two functions in terms of their individual Fourier transforms:

\begin{lemma}[Lemma 1 in \cite{diaconis}]
\label{lemm:conv}
    Let $G$ be a group and let $f_1$, $f_2$ be functions acting on $G$. Then for any representation $\rho$ we have 
    $$\rho(f_1 * f_2) = \rho(f_1) \rho(f_2).$$
\end{lemma}

By induction, \Cref{lemm:conv} generalizes for the convolution of more than two functions. Concluding this section, we present a series of standard lemmas that we use in our analysis:

\begin{lemma}[Plancherel formula, Exercise 3.32 in \cite{fulton2013representation}]
\label{lemm:planch}
    Let $G$ be a group and $f, h: G \to \CC$. Then 
    $$\sum_{g\in G} f(g)\overline{h(g)} = \frac{1}{|G|}\sum_{\rho \in \hat{G}} d_\rho \Tr[\rho(f) \rho(h)^*].$$
     The sum of the right-hand side is taken over all irreducible unitary representations of $G$.
\end{lemma}

\begin{lemma}[Schur's lemma, Proposition 4 in \cite{serre1996linear}]
\label{lemm:schur}
   Let $\rho_1: G \to \GL(V_1)$ and $\rho_2: G \to \GL(V_2)$ be irreducible representations of $G$, and let $M$ be a matrix that $\rho_2(g)M = M\rho_1(g)$ for all $g\in G$. Then:
   \begin{enumerate}
   \item If $\rho_1$ and $\rho_2$ are not isomorphic, we have $M = 0$
   \item If $V_1 = V_2$ and $\rho_1= \rho_2$, then $M = cI$ for some $c\in \CC$.
   \end{enumerate}
\end{lemma}

Schur's Lemma gives rise to a useful corollary in the case where a representation $\rho: G \to \GL(V)$ can be decomposed into a direct sum of non-isomorphic irreps $(\rho_1, V_1),\dots, (\rho_k, V_k)$, i.e., $\rho=\rho_1 \oplus \dots \oplus \rho_k$ and $\rho_i \not\simeq \rho_j$ for all $i\neq j$:

\begin{corollary}[Schur's Lemma for reducible representations]
\label{corr:schur_gen}
    Let $\rho : G \to \GL(V)$ be a representation of $G$ that decomposes as a direct sum of non-isomorphic irreps $(\rho_1, V_1), \dots, (\rho_k, V_k)$, and let $M$ be a matrix such that $\rho(g)M=M\rho(g)$ for all $g\in G$. Then 
    $$M = \begin{bmatrix}c_1 I_{d_{\rho_1}} & 0 &\dots &0 \\ 0&  c_2I_{d_{\rho_2}} & \dots &0 \\ \vdots & & \ddots & \vdots \\ 
    0 & \dots &0 & c_k I_{d_{\rho_k}}\end{bmatrix}$$
    for some constants $c_1,\dots, c_k \in \CC$.
\end{corollary}

\begin{proof}
Recall that by \Cref{def:direct_sum} we have 
$$\rho(g)=\begin{bmatrix}
    \rho_1(g) & \dots & 0 \\
    \vdots & \ddots & \vdots \\
    0 & \dots  & \rho_k(g) 
\end{bmatrix}$$
The condition $\rho(g) M = M \rho(g)$ thus translates to 
$$\begin{bmatrix}
    M_{11}  & \dots & M_{1k} \\ 
    \vdots  & \ddots & \vdots \\ 
    M_{k1} & \dots & M_{kk}
\end{bmatrix} \begin{bmatrix}
    \rho_1(g) & \dots & 0 \\
    \vdots & \ddots & \vdots \\
    0 & \dots  & \rho_k(g) 
\end{bmatrix} = \begin{bmatrix}
    \rho_1(g) & \dots & 0 \\
    \vdots & \ddots & \vdots \\
    0 & \dots  & \rho_k(g) 
\end{bmatrix} \begin{bmatrix}
    M_{11}  & \dots & M_{1k} \\ 
    \vdots  & \ddots & \vdots \\ 
    M_{k1} & \dots & M_{kk}
\end{bmatrix}$$
where $M_{ij}$ are $d_{\rho_i} \times d_{\rho_j}$ blocks of $M$. This gives $M_{ij} \rho_j(g) = \rho_i(g) M_{ij}$ for all $i,j\in [k]$. The result trivially follows by applying Schur's Lemma (\Cref{lemm:schur}).
\end{proof}
\begin{lemma}[Lemma 5 in \cite{diaconis}]\label{lem:transformer_identity}
    \label{lemm:fourier_form}
    Let $G$ be a group and $\rho$ an irreducible representation of $G$. Let $f: G \to \CC$ be a class function, i.e., it is constant on each conjugacy class. Let $f_i$ be the value of $f$ on the $i$-th conjugacy class, $n_i$ the cardinality of the $i$-th conjugacy class, and $\chi_i$ the value of the character of $\rho$ on the $i$-th conjugacy class. Then 
    $$\rho(f) = CI\quad \text{ where } \quad C=\frac{1}{d_\rho} \sum_i f_i n_i\chi_i.$$
    The sum is taken over distinct conjugacy classes.
\end{lemma}

The last result, which we prove, characterizes the Fourier transforms of the uniform probability distribution on a group $G$ with respect to irreducible representations.

\begin{lemma}
\label{lemm:fourier_unif}
    Let $G$ be a group and $U: G \to \CC$ be the uniform probability distribution over $G$, namely $U(g) = 1/|G|$ for all $g\in G$. Then 
    $$\rho (U) = \begin{cases}
        1 & \text{if } \rho = \triv \\
        0_{d_\rho \times d_\rho} & \text{if }\rho \in \hat{G}\setminus \{\triv\}.
    \end{cases}$$
\end{lemma}
\begin{proof}
    The case $\rho = \triv$ follows by the definition of the trivial representation. For an irrep $\rho \neq \triv$ we use Schur's lemma (\Cref{lemm:schur}). Notice that for all $g \in G$ we have 
    $$\rho(U) \rho(g) = \rho(g)\rho(U) = \frac{1}{|G|}\sum_{g\in G} \rho(g).$$
    Hence, Schur's lemma asserts that $\rho(U) = cI_{d_\rho}$ for some $c \in \CC$. Taking the trace of both sides, we get 
    $$\frac{1}{|G|} \sum_{g\in G} \chi_\rho(g) = cd_\rho.$$
    The left-hand side of the above expression is precisely the product of the characters of $\rho$ and $\triv$, $\langle \chi_\rho, \chi_{\triv}\rangle$. By the orthogonality relations of characters, $\langle \chi_\rho, \chi_{\triv}\rangle = 0$ and so $c = 0$. This concludes the proof.
\end{proof}

\subsection{Representations of the symmetric group} 
\label{sub:s_n}

Although the results of the previous section apply to any finite group, this work exclusively focuses on the symmetric group $S_N$ and the product group $S_N \times S_N$. Fortunately, the representation theory of symmetric groups is well studied, and we can draw on a wealth of established results. In this section, we summarize the key definitions and facts about the representations of $S_N$. Most of the results presented as ``Facts" are taken from Chapter 4 of \cite{fulton2013representation}. Alongside each statement, we provide a reference to the corresponding passage. Recall the definition of the symmetric group from \Cref{ex:grp}:

\begin{definition}[Symmetric group]
    The set $S_N$ of all permutations of $[N]$ (i.e., bijections $f: [N] \to [N]$), equipped with function composition, forms a group of order $N!$ called the \emph{symmetric group} on $N$ elements.
\end{definition}

It can be shown that irreducible representations of $S_N$ are in one-to-one correspondence with the partitions of $N$. A \emph{partition} of a positive integer $N$ is a tuple $(\lambda_1, \dots, \lambda_k) \in \NN^k$ such that $\lambda_1\geq \lambda_2\geq\dots\geq \lambda_k \geq1$ and $\lambda_1+\dots+\lambda_k = N$. We can now characterize the irreps of $S_N$ as follows:

\begin{fact}[p. 44 in \cite{fulton2013representation}]
\label{fact:corresp}
    Each irrep of $S_N$ corresponds to a partition of $N$. Conversely, each partition of $N$ corresponds to an irrep of $S_N$. The trivial representation of $S_N$ corresponds to the partition $(N)$.
\end{fact}

The \emph{standard representation} of $S_N$, denoted by $\std$, is the representation corresponding to the partition $(N-1, 1)$. In what follows, we will occasionally denote irreps by their corresponding partitions. For example, we may write $(N-1,1)$ instead of $\std$. 

If we let $p(N)$ denote the number of partitions of $N \in \NN$, \Cref{fact:corresp} and \Cref{thm:tensor_prod} give the following:

\begin{remark}
\label{rem:num_irreps}
    The number of irreducible representations is $p(N)$ for $S_N$ and $p(N)^2$ for $S_N \times S_N$.
\end{remark}

The following fact states that irreducible representations of $S_N$ can be defined over the reals.

\begin{fact}[p. 46 in \cite{fulton2013representation}]
\label{fact:real_rep}
    Each irreducible representation of $S_N$ can be defined over the rationals. In particular, we can define every irreducible representation $\rho$ of $S_N$ such that $\rho(g)$ is an orthogonal matrix for every $g \in G$.
\end{fact}

The dimension of an irrep can be determined solely by the corresponding partition, as shown in the following fact:

\begin{fact}[Equation 4.11 in \cite{fulton2013representation}]
    \label{fact:irrep_dim}
    Let $\rho =(\lambda_1,\dots,\lambda_k)$ be an irrep of $S_N$. Then 
    $$d_\rho = \frac{N!}{l_1!\cdots l_k!} \prod_{i<j}(l_i-l_j),$$
    where $l_i = \lambda_i+k-i$.
\end{fact}

Another quantity of interest is the \emph{character ratio} of transpositions. For an irrep $\rho$, it is defined as $r(\rho) =\chi_{\rho}(\tau)/d_{\rho}$, where $\chi_{\rho}(\tau)$ is the character of $\rho$ evaluated at any transposition.\footnote{Note that this is well-defined since transpositions form a conjugacy class.} The following fact gives a closed-form expression for $r(\rho)$ that only depends on the corresponding partition: 

\begin{fact}[Lemma 7 in \cite{diaconis}]
\label{fact:char_ratio}
    Let $\rho = (\lambda_1,\dots,\lambda_k)$ be an irrep of $S_N$. The character ratio of transpositions is given by 
    $$r(\rho) = \frac{1}{N(N-1)}\sum_{j=1}^k \bigl[(\lambda_j-j)(\lambda_j-j+1) - j(j-1)\bigr].$$
\end{fact}

\begin{example} Using \Cref{fact:char_ratio} we can obtain the following:
    \begin{enumerate}[i)]
    \item It follows by the definition of the trivial representation that $r(\triv) = 1$. It is easy to see that the expression given in \Cref{fact:char_ratio} gives the same result. 
    \item For the standard representation $\std = (N-1,1)$, we obtain $r(\std) =\frac{N-3}{N-1}$.
    \end{enumerate}
\end{example}

We now introduce the permutation representation of $S_N$, a useful non-irreducible representation of $S_N$.

\begin{definition}[Permutation representation] Let $\{\vect{e_1},\dots,\vect{e_N}\}$ be the standard basis of $\CC^N$. The \emph{permutation representation}\footnote{In some textbooks it may also be referred to as the defining representation of $S_N$.} of $S_N$, denoted by $\operatorname{perm}$, is the homomorphism $\rho: S_N \to \GL(\CC^N)$ defined by 
$$\rho(g)(\vect{e_i}) = \vect{e_{g(i)}}$$
for all $g \in S_N$ and $i\in [N]$.
\end{definition}

We leave the verification that $\perm$ is indeed a valid representation of $S_N$ as an (easy) exercise. In essence, the permutation representation of an element $g \in S_N$ is the linear map that acts by permuting the standard basis vectors of $\CC^N$ according to $g$. The permutation representation is closely related to the number of fixed points of the elements of $S_N$. This is made precise in \Cref{lemm:perm_char}, following the definition of fixed points.

\begin{definition}
    Let $g \in S_N$ be a permutation. We say that $i \in [N]$ is a \emph{fixed point} of $g$ if $g(i) = i$. The set of fixed points of $g$ is denoted by $\fix(g) = \{i \in [N]: g(i)=i\}$.
\end{definition}

\begin{lemma}
\label{lemm:perm_char}
    The character of $\perm$ on any $g \in S_N$ is equal to the number of fixed points of $g$. Namely, $\chi_{\perm}(g) = |\fix(g)|$ for all $g\in S_N$.
\end{lemma}

\begin{proof}
    Observe that, by definition, when expressed in the standard basis, $\rho(g)$ is a permutation matrix. The diagonal element in the $(i,i)$-th coordinate is equal to $1$ \textit{if and only if} $\rho(g)(\vect{e}_i)=\vect{e_i}$, which happens \textit{if and only if} $g(i) = i$, or equivalently $i \in \fix(g)$. Hence, 
    $$\chi_{\perm}(g) = \sum_{i=1}^N \rho(g)_{ii} = |\fix(g)|,$$
    which concludes the proof.
\end{proof}

Since $\perm$ is not irreducible, Maschke's theorem guarantees that it is isomorphic to a direct sum of irreps. The decomposition is given in the following fact:

\begin{fact}[p. 55 in \cite{fulton2013representation}]
\label{fact:perm_decomp}
    The permutation decomposition of $S_N$ is isomorphic to the direct sum of the trivial and the standard representation of $S_N$, namely
    $\perm \simeq \triv \oplus \std$.     Consequently, 
    $$\chi_{\perm}(g) = \chi_{\triv}(g) + \chi_{\std}(g) = 1 + \chi_{\std}(g)$$
    for all $g\in S_N$.
\end{fact}
\section{Proof of \Cref{thm:p_agree}}
\label{app:proof_pagree}

In this section, we provide the full proof of \Cref{thm:p_agree}. The proof relies on a technical lemma regarding the Fourier transform of the function on $S_N \times S_N$ given by $(g,h) \mapsto |\fix(h^{-1}g)|$, which we state and prove first.

\begin{lemma}[Fourier transform of $\fix$]
\label{lemm:fourier_fix}
Let $G=S_N \times S_N$ and $f: G \to \CC$ be given by 
$f(g,h)=|\fix(h^{-1}g)|$. Then for any non-trivial irreducible representation $\Pi$ of $G$, we have
$$\Pi(f) = \begin{cases}
    \frac{(N!)^2}{N-1} P_{\operatorname{diag}} & \text{if } \; \Pi = \std \otimes \std \\
     0_{d_\Pi \times d_\Pi} & \text{otherwise }
\end{cases}$$
where $P_{\operatorname{diag}}$ is the orthogonal projection onto the diagonal subspace  $\operatorname{span}\left\{\sum_{i=1}^{N-1} \vect{e_i} \otimes \vect{e_i}\right\}$.
\end{lemma}

\begin{proof}
    By \Cref{thm:tensor_prod}, we can write $\Pi = \rho \otimes \sigma$ with at least one of the irreps $\rho, \sigma \in \hat{S}_N$ being non-trivial. The Fourier transform of $f$ is thus given by
$$\Pi(f) = \sum_{g,h\in S_N} |\fix(h^{-1}g)|\rho(g)\otimes \sigma(h)$$
 Performing the change of variables $g'=h^{-1}g$, the above sum becomes
\begin{align*} \Pi(f) &= \sum_{g,h \in S_N} |\fix(g)| \rho(hg) \otimes \sigma(h)  \\ 
&= \sum_{g, h \in S_N} |\fix(g)| (\rho(h)\otimes \sigma(h))(\rho(g)\otimes I_{d_\sigma}) \\ &=
\sum_{g\in S_N}
 |\fix(g)|\left\{ \sum_{h \in S_N} \rho(h)\otimes\sigma(h) \right\}(\rho(g) \otimes I_{d_\sigma}) \\
 &= \left(\sum_{h\in S_N} \rho(h)\otimes\sigma(h)\right)\left(\sum_{g\in S_N} |\fix(g)| \rho(g)\right) \otimes I_{d_\sigma}
 \end{align*}

Since $|\fix(\cdot)|$ is the character of the permutation representation of $S_N$ (\Cref{lemm:perm_char}), and hence a class function on $S_N$, by \Cref{lemm:fourier_form}, the second sum is equal to $C I_{d_\rho}$ where 
$$C=\frac{1}{d_\rho} \sum_{g \in S_N} |\fix(g)|\chi_\rho(g) = \frac{N!}{d_\rho}(\langle \chi_{\triv}, \chi_{\rho}\rangle + \langle \chi_{\std}, \chi_{\rho}\rangle).$$
The last equality follows from the fact that the permutation representation decomposes as the direct sum of the trivial representation and the standard representation (\Cref{fact:perm_decomp}). In particular, character orthogonality (\Cref{thm:orhtogonality}) yields $C \neq 0$ if and only if $\rho \in\{ \triv,\std\}$.

For the first sum, let
$$M = \sum_{h\in S_N} \rho(h)\otimes \sigma(h).$$
Indexing $M$ by the indices of the products we get 
$$M_{(i,j),(k,l)} = \sum_{h\in S_N} \rho(h)_{ij} \sigma(h)_{kl} = \sum_{h\in S_N} \rho(h)^\top _{ji} \sigma(h)_{kl}.$$
By the Great Orthogonality Theorem (\Cref{thm:GOT}) and the fact that $\rho$ can be taken such that $\rho(h)$ is real for every $h \in S_N$ (\Cref{fact:real_rep}), $M_{(i,j),(k,l)}$ vanishes, unless $\rho=\sigma$, in which case it works out to 
$$M_{(i,j),(k,l)} = \frac{N!}{d_\rho}\delta_{jl}\delta_{ik}.$$
In matrix form, when $\rho = \sigma$ 
$$M = \frac{N!}{d_\rho} \sum_{i,j=1}^{d_\rho} E_{ij}\otimes E_{ij},$$
where $E_{ij} = \vect{e_i}\vect{e_j}^\top \in \RR^{d_{\rho} \times d_{\rho}}$.
Putting everything together, we obtain that $\Pi(f)$ does not vanish \textit{if and only if} $\rho=\sigma$ and $\rho \in \{\triv, \std\}$. If $\rho=\sigma=\triv$ then $\Pi=\triv$, and so the only non-trivial irrep $\Pi$ for which the Fourier transform does not vanish is $\Pi = \std \otimes \std$. In that case, combining the above calculations with the fact that $d_{\std \otimes \std} = d_{\std}^2 = (N-1)^2$, we obtain 
$$\Pi(f) = \frac{(N!)^2}{(N-1)^2} \sum_{i,j=1}^{N-1} E_{ij} \otimes E_{ij}.$$
To see why this is equal to the required expression, let $\vect{v} = \sum_{i=1}^{N-1} \vect{e_i} \otimes \vect{e_i}$ and observe that 
\begin{align*}
    \frac{1}{N-1}\sum_{i,j=1}^{N-1} E_{ij} \otimes E_{ij} &= \frac{1}{\|\vect{v}\|^2} \sum_{i,j=1}^{N-1} (\vect{e_i}\vect{e_j}^\top) \otimes (\vect{e_{i}}\vect{e_j}^\top) = \frac{1}{\|\vect{v}\|^2} \sum_{i,j=1}^{N-1} (\vect{e_i} \otimes \vect{e_i})(\vect{e_j} \otimes \vect{e_j})^\top \\ 
    &=  \frac{1}{\|\vect{v}\|^2} \left(\sum_{i=1}^{N-1} \vect{e_i} \otimes \vect{e_i}\right)\left(\sum_{j=1}^{N-1} \vect{e_j} \otimes \vect{e_j}\right)^\top = \frac{1}{\|\vect{v}\|^2} \vect{v}\vect{v}^\top.
\end{align*}
The latter expression is precisely equal to $P_{\operatorname{diag}}$, the orthogonal projection matrix onto $\operatorname{span}\{\vect{v}\}$, thus concluding the proof.
\end{proof}

We are now ready to prove \Cref{thm:p_agree}, which we restate for convenience:

\pagree*

\begin{proof}
The analysis of $P_{\operatorname{agree}}$ requires understanding the convergence of the joint distribution of $(P_{\mathrm{w}_T}, P'_{\mathrm{w}_T})$ on the product group $G = \Sn \times \Sn$. We condition on the final state $(g,h) \in G$ of the random walk after $T$ steps. Let $P_T(g, h) = \Prob_{\w_T}(P_{\mathrm{w}_T}=g, P'_{\mathrm{w}_T}=h)$. The number of states on which two permutations $g$ and $h$ agree is the number of fixed points of $h^{-1}g$, denoted $\fix(h^{-1}g)$. Thus, we write: 
\begin{align*}
P_{\operatorname{agree}} &= \sum_{(g,h) \in G} \Prob(P_{\mathrm{w}_T}=g, P'_{\mathrm{w}_T}=h) \cdot \Prob_{X \sim \U(\Q)}(g(X)=h(X) \mid P_{\mathrm{w}_T}=g, P'_{\mathrm{w}_T}=h) \\
&= \sum_{(g,h) \in G} P_T(g, h) \cdot \frac{|\{X \in \Q \mid g(X)=h(X)\}|}{N} \\
&= \sum_{(g,h) \in G} P_T(g, h) \cdot \frac{|\fix(h^{-1}g)|}{N}. 
\end{align*}
The distribution $P_T$ of the random walk on $G$ can be written in terms of the uniform distribution over the elements of $G$, $P_U(g,h)= 1/|G|$ and a residual term $\operatorname{err}(g, h)$, namely $P_T(g,h) = P_U(g,h) + \operatorname{err}(g,h)$,
where $\operatorname{err}(g,h) = P_T(g,h) - P_U(g,h)$.
Therefore, we get 
\begin{align}
\label{eq:p_agree}
    P_{\operatorname{agree}} & = \frac{1}{N} \sum_{(g,h) \in G} \left( \frac{1}{(N!)^2} + \operatorname{err}(g,h) \right) \cdot \fix(h^{-1}g) \notag\\ 
    & = \underbrace{\frac{1}{N \cdot (N!)^2} \sum_{(g,h) \in G} \fix(h^{-1}g)}_{\text{Main Term}} + \underbrace{\frac{1}{N}\sum_{(g,h) \in G} \left( P_T(g,h) - P_U(g,h) \right) \cdot \fix(h^{-1}g)}_{\text{Error Term}}.
\end{align}
For the main term of \Cref{eq:p_agree} notice that for fixed $g\in S_N$ we have 
$$\sum_{h \in S_N} \fix(h^{-1}g) = \sum_{h \in S_N} \fix(h)$$
since the map $h \mapsto h^{-1}g$ is a bijection. Hence, we have
\begin{equation}
\label{eq:main_term}
\text{Main Term} = \frac{1}{N\cdot (N!)^2} \left(\sum_{g\in S_N} \fix(g)\right)^2 = \frac{1}{N \cdot (N!)^2} \cdot (N!)^2 = \frac{1}{N},
\end{equation}
where the second-to-last equality follows from the fact that 
\begin{align*}
    \sum_{g \in \Sn} \fix(g) & = \sum_{g \in \Sn} \sum_{X\in \Q} \indic_{\{g(X)=X\}} = \sum_{X\in \Q} \sum_{g \in \Sn} \indic_{\{g(X)=X\}} \\
    &= \sum_{X\in \Q} (N-1)! = N \cdot (N-1)! = N!.
\end{align*}
We now turn our attention to the error term of \Cref{eq:p_agree}, given by 
\[
    \frac{1}{N}\sum_{(g,h) \in G} \left( P_T(g,h) - P_U(g,h) \right) \cdot \fix(h^{-1}g).
\]
By letting $f(g, h) = \fix(h^{-1}g)$ and using the Plancherel formula (\Cref{lemm:planch}), the error term is equal to:
\begin{equation} \label{eq:fourier_sum}
    \frac{1}{N}\cdot \frac{1}{|G|} \sum_{\Pi \in \hat{G}} d_\Pi \Tr\left( \Pi(\operatorname{err}) \Pi(f)^* \right).
\end{equation}
Since the Fourier transform is linear and $\Pi(P_U)=0$ for any non-trivial irrep $\Pi$ (\Cref{lemm:fourier_unif}), and $\Pi(P_T) = M_\Pi^T$ where $M_\Pi$ is the Fourier transform of the single-step distribution (\Cref{lemm:conv}), the expression simplifies to a sum over non-trivial irreps:
\[
    \frac{1}{N}\cdot \frac{1}{|G|} \sum_{\Pi \neq \triv} d_\Pi \Tr\left( M_\Pi^T \Pi(f)^* \right).
\]
Finally, using \Cref{lemm:fourier_fix}, we see that all terms of the sum vanish except for the contribution of the irrep $\std \otimes \std$, which gives
\begin{align}
\label{eq:error_term}
\text{Error Term} &= \frac{1}{N\cdot (N!)^2}\cdot \frac{(N!)^2}{N-1}\cdot (N-1)^2 \Tr(M_{\Pi_0}^T P_{\operatorname{diag}}) \notag\\ 
&= \frac{N-1}{N}\cdot \frac{1}{\|\vect{v}\|^2} \Tr\left(M_{\Pi_0}^T \vect{v}\vect{v}^\top\right) \notag \\
&= \frac{1}{N} \vect{v}^\top M_{\Pi_0}^T \vect{v},
\end{align}
where the last equality follows from the cyclic property of the trace and the fact that $\|\vect{v}\|^2 = N-1$. Substituting \Cref{eq:main_term} and \Cref{eq:error_term} into \Cref{eq:p_agree} we obtain 
$$P_{\operatorname{agree}} = \frac{1}{N} + \frac{1}{N}\vect{v}^\top M_{\Pi_0}^T \vect{v},$$
as required.
\end{proof}
\section{Proofs from  \Cref{sec:randomized_family}}
\label{app:proofs_spectral}

In this section, we provide detailed proofs of the key results presented in \Cref{sec:randomized_family}: \Cref{thm:spectral_bound_random}, \Cref{lem:prob_of_agreement}, and \Cref{thm:randomized_t}. For the proof of \Cref{lem:prob_of_agreement}, we make use of the following version of the Matrix Bernstein inequality for Hermitian matrices:

\begin{lemma}[Matrix Bernstein inequality, Theorem 6.6.1 in \cite{MAL-048}]
\label{lem:bernstein}
Let $Z_1, \dots, Z_m$ be independent random $d \times d$ Hermitian matrices with $\E[Z_i] = 0$ and $\norm{Z_i} \le B$ a.s. Let $v = \norm{\sum_{i=1}^m \E[Z_i^2]}$ be the matrix variance parameter. Then for any $t > 0$:
\[ \Prob\left( \norm{\sum_{i=1}^m Z_i} > t \right) \le d \cdot \exp\left( \frac{-t^2}{2v + \frac{2}{3}Bt} \right) \]
\end{lemma}

The statements of the results are restated below for convenience. 

\spectralbound*

\begin{proof}
Our proof can be summarized in two key steps: first, we compute the operator norm of the expectation of the Fourier transform of the single-step probability distribution; subsequently, we use concentration arguments to show that for the particular choice of $k$ and $M$, the norm of the actual operator concentrates around this norm.

Let $\rho = \std$ and take $\delta_i, \delta_j \in \F$ with $i\neq j$. To avoid cluttering, we denote the state-transition function for each character $\tau \in \Sigma_k$ by $\delta_i(\tau)$. The expected Fourier transform of the single-step probability distribution for the joint walk according to $\delta_i$ and $\delta_j$ is given by 
\begin{align}
\mathbb{E}[M_{\Pi_0}^{i,j}] &= \frac{1}{|\Sigma_k|}\sum_{\tau \in \Sigma_k} \mathbb{E}[\rho(\delta_i(\tau)) \otimes \rho(\delta_j(\tau))]\notag \\ 
&=\frac{1}{k|\Sigma_1|} \sum_{j=1}^k \sum_{\tau \in \Sigma_1} \frac{1}{4} \left( \rho(\tau)\otimes\rho(\tau) + \rho(\tau)\otimes I_{d_\rho} + I_{d_\rho}\otimes\rho(\tau) + I_{d_\rho} \otimes I_{d_\rho} \right) \notag\\
&= \frac{1}{4|\Sigma_1|} \biggl[ \left(\sum_{\tau \in \Sigma_1} \rho(\tau)\otimes\rho(\tau)\right) + \left(\sum_{\tau \in \Sigma_1}\rho(\tau)\right)\otimes I_{d_\rho} + I_{d_\rho}\otimes\left(\sum_{\tau \in \Sigma_1}\rho(\tau)\right) \notag\\&\quad\quad\quad\quad\;\;+ |\Sigma_1|I_{d_{\rho}^2} \biggr] \notag\\
&= \frac{1}{4|\Sigma_1|} \left\{\left(\sum_{\tau \in \Sigma_1} \rho(\tau) \otimes \rho(\tau)\right) + 2c_{\rho} I_{d_\rho^2} + |\Sigma_1|I_{d_\rho^2} \right\}
\end{align}
where $c_{\rho} = |\Sigma_1|r(\rho)$ is given by an application of \Cref{lemm:fourier_form}.
Next, let 
$$S = \sum_{\tau \in \Sigma_1} \rho(\tau) \otimes \rho(\tau)$$
and consider $\pi(g) = \rho(g) \otimes \rho(g)$ for all $g\in S_N$. By Exercise 4.19 in \cite{fulton2013representation}, for $N \geq 4$, $\pi$ is a (non-irreducible) representation of $S_N$ that decomposes as a direct sum of non-isomorphic irreps
\begin{equation}
\label{eq:decomp_irrep}
    \pi \simeq \triv \oplus \std \oplus (N-2,2) \oplus (N-2,1,1)
\end{equation}
where we identify representations by their corresponding partition. A short calculation shows that for all $g\in S_N$ we have \begin{align*}
\pi(g)S\pi(g)^{-1} & = \sum_{\tau \in \Sigma_1} \rho(g \tau g^{-1}) \otimes \rho(g \tau g^{-1}).
\end{align*}
and since the set of transpositions is a conjugacy class (\Cref{def:conj}), the sum above is equal to $S$. Hence, $\pi(g)S = S\pi(g)$ for all $g\in G$, and by the generalization of Schur's Lemma given in \Cref{corr:schur_gen}, we obtain that (under a change of basis) $S$ has a block diagonal form. 
In particular, we have that\footnote{While \Cref{corr:schur_gen} does not require a change of basis, in this case it is induced by the isomorphism between the representations in \Cref{eq:decomp_irrep}. From this point onward, we assume that $S$ is expressed in the new basis.} 
$$S = c_{\triv} I_{d_{\triv}} \oplus c_{\std} I_{d_{\std}} \oplus c_{(N-2,2)} I_{d_{(N-2,2)}} \oplus c_{(N-2,1,1)} I_{d_{(N-2,1,1)}}
$$
The coefficients $c_a$ can be calculated by restricting $S$ to the underlying vector space $V_a$ corresponding to each direct summand of \Cref{eq:decomp_irrep} and taking traces. Hence, for $a \in \{\triv, \std, (N-2,2), (N-2,1,1)\}$ we find
$$c_a d_a = \sum_{\tau \in \Sigma_1} \chi_a(\tau)$$
which, given that $\Sigma_1$ is the set of transpositions, simplifies to 
$$c_a = |\Sigma_1|\cdot \frac{\chi_a(\tau)}{d_a} = |\Sigma_1|r(a)$$
where $\chi_a(\tau)$ is the character of $a$ on transpositions. The eigenvalues of $\mathbb{E}[M_{\Pi_0}^{i,j}]$ are thus given by $\frac{1}{4|\Sigma_1|}(c_a + 2c_\rho + |\Sigma_1|)$ for $a \in \{\triv, \std, (N-2, 2), (N-2,1,1)\}$. Substituting the values for $c_a$, we find that the eigenvalues are given by $\frac{1}{4}(r(a)+2r(\rho)+1)$. The value of the character ratios can be computed by invoking \Cref{fact:char_ratio}:
\begin{itemize}
    \item For $a = \triv$, we find $r(a) = 1$.
    \item For $a = \rho=\std$, we find $r(a) = \frac{N-3}{N-1}$.
    \item  For $a = (N-2, 2)$, we find $r(a)=\frac{N-4}{N}$.
    \item For $a = (N-2,1,1)$, we find $r(a)=\frac{N-5}{N-1}$.
\end{itemize}

By the calculations above, we have 
\begin{equation}
\label{eq:expected_spec}
\operatorname{spec}(\mathbb{E}[M_{\Pi_0}^{i,j}]) = \left\{\frac{N-2}{N-1}, \frac{2N-5}{2(N-1)},\frac{N^2-3N+1}{N(N-1)},\frac{N-3}{N-1}\right\}.
\end{equation}
Since the expectation is Hermitian (representations can always be chosen to be unitary, see \Cref{subsub:rep}), the operator norm of $\mathbb{E}[M_{\Pi_0}^{i,j}]$ is equal to its maximum absolute eigenvalue and hence 
\begin{equation}
    \label{eq:m_pi_norm}
    \|\mathbb{E}[M_{\Pi_0}^{i,j}]\|_2 = \frac{N-2}{N-1} = 1-\frac{1}{N-1}.
\end{equation}

For the last step of the proof, we will choose appropriate values for $k$ and $M$ and use the Matrix Bernstein inequality to derive a high probability bound on the operator norm $\|M_{\Pi_0}^{i,j}\|_2$ for a randomized $(k,M)$-shuffle family. Let $\F$ be such a family (the values $k$ and $M$ will be determined later) and fix $\delta_i, \delta_j \in \F$. From the triangle inequality, we get
\begin{equation}
    \label{eq:trian_eq}
    \|M_{\Pi_0}^{i,j}\|_2 \leq \|\mathbb{E}[M_{\Pi_0}^{i,j}]\|_2 + \|M_{\Pi_0}^{i,j} -\mathbb{E}[M_{\Pi_0}^{i,j}]\|_2.
\end{equation}
For every $\tau \in \Sigma_k$ we let
$$Z_{\tau}^{i,j} = \rho(\delta_i(\tau)) \otimes \rho(\delta_j(\tau)) - \mathbb{E}[\rho(\delta_i(\tau)) \otimes \rho(\delta_j(\tau))]$$
and rewrite 
\begin{equation}
\label{eq:rewrite}
M_{\Pi_0}^{i,j} -\mathbb{E}[M_{\Pi_0}^{i,j}] =\frac{1}{|\Sigma_k|}\sum_{\tau \in \Sigma_k} Z_{\tau}^{i,j}.
\end{equation}
By construction, the $Z_{\tau}^{i,j}$ are independent, Hermitian,\footnote{To see why, notice that since transpositions satisfy $\tau^2 = \id$ we have $\delta_i(\tau)^2=\delta_j(\tau)^2 = \id$. Since $\rho$ preserves the group operation and can be chosen to be unitary, $\rho(\delta_i(\tau))$ and $\rho(\delta_j(\tau))$ are Hermitian. The Kronecker product and the expectation of Hermitian (random) matrices are Hermitian.} zero-mean matrices that satisfy 
$$\|Z_{\tau}^{i,j}\|_2 \leq \|\rho(\delta_i(\tau)) \otimes \rho(\delta_j(\tau))\|_2 + \|\mathbb{E}[\rho(\delta_i(\tau)) \otimes \rho(\delta_j(\tau))]\|_2 \leq 2$$
We now analyze the variance parameter 
$
    v = \norm{\sum_{\tau \in \Sigma_k} \E\left[\left(Z_{\tau}^{i,j}\right)^2\right]}
$.
For every $\tau \in \Sigma_k$ we have
\begin{align*}
\E\left[\left(Z_{\tau}^{i,j}\right)^2\right] &= \E[(\rho(\delta_i(\tau)) \otimes \rho(\delta_j(\tau)))^2] - \E[\rho(\delta_i(\tau))\otimes \rho(\delta_j(\tau))]^2 \\ &= I_{d_{\rho}^2} - \E[\rho(\delta_i(\tau))\otimes \rho(\delta_j(\tau))]^2
\end{align*}
where the last equality follows from the fact that transpositions in $S_N$ satisfy $\tau^2 = \id$ and hence $\delta_i(\tau)^2 = \id$ (regardless of the realization). Consequently, 
\begin{align*}
    (\rho(\delta_i(\tau)) \otimes \rho(\delta_j(\tau)))^2 &= \rho(\delta_i(\tau))^2 \otimes \rho(\delta_j(\tau))^2 = \rho(\delta_i(\tau)^2) \otimes\rho(\delta_j(\tau)^2) \\&= \rho(\id)\otimes\rho(\id) = I_{d_{\rho}^2}.
\end{align*}
By an application of Jensen's inequality we get $\|\E[\rho(\delta_i(\tau))\otimes \rho(\delta_j(\tau))]\|_2 \leq 1$, and since the expectation is Hermitian, we can deduce that $I_{d_\rho} - \E[\rho(\delta_i(\tau))\otimes \rho(\delta_j(\tau))]^2$ is positive semidefinite with eigenvalues bounded by $1$. Hence, 
\begin{align*}
    v &= \left\|\sum_{\tau \in \Sigma_k} \E\left[\left(Z_{\tau}^{i,j}\right)^2\right]\right\|_2 = \left\|\sum_{\tau \in \Sigma_k} (I_{d_{\rho}^2} - \E[\rho(\delta_i(\tau))\otimes \rho(\delta_j(\tau))]^2)\right\|_2  \\ 
    &\leq \sum_{\tau \in \Sigma_k} \|I_{d_{\rho}^2} - \E[\rho(\delta_i(\tau))\otimes \rho(\delta_j(\tau))]^2\|_2 \leq \sum_{\tau \in \Sigma_k} 1 = |\Sigma_k|.
\end{align*}
We call a randomized $(k, M)$-shuffle family $\F$ ``bad" if there exist distinct semiautomata $\delta_i, \delta_j \in \F$ such that $\|M_{\Pi_0}^{i,j} - \E[M_{\Pi_0}^{i,j}]\|_2 > \frac{1}{2N}$, and ``good" otherwise. Notice that by \Cref{eq:trian_eq}, and the derivation in \Cref{eq:m_pi_norm}, whenever $\F$ is a good family, every distinct pair of semiautomata $\delta_i, \delta_j \in \F$ satisfy
$$\|M_{\Pi_0}^{i,j}\|_2 \leq \|\mathbb{E}[M_{\Pi_0}^{i,j}]\|_2 + \|M_{\Pi_0}^{i,j} -\mathbb{E}[M_{\Pi_0}^{i,j}]\|_2 < 1-\frac{1}{N} + \frac{1}{2N} = 1 -\frac{1}{2N},$$
as required. Hence, it suffices to consider the probability of generating a bad family. By the union bound 
\begin{equation*}
\Prob(\F \text{ is bad}) \leq \sum_{\delta_i, \delta_j \in \F} \Prob\left(\|M_{\Pi_0}^{i,j} - \E[M_{\Pi_0}^{i,j}]\|_2 > \frac{1}{2N}\right).
\end{equation*}
Each term of the summation above can be upper-bounded by an application of the matrix Bernstein inequality. Indeed, by \Cref{eq:rewrite} and matrix Bernstein we have 
\begin{align*}
    \Prob\left(\|M_{\Pi_0}^{i,j} - \E[M_{\Pi_0}^{i,j}]\|_2 > \frac{1}{2N}\right) &= \Prob\left(\left\|\frac{1}{|\Sigma_k|}\sum_{\tau \in \Sigma_k} Z_{\tau}^{i,j}\right\|_2 \geq \frac{1}{2N}\right) \\&= \Prob\left(\left\|\sum_{\tau \in \Sigma_k} Z_{\tau}^{i,j}\right\|_2 \geq \frac{|\Sigma_k|}{2N}\right) \\&\leq d_\rho^2 \cdot \exp\left(\frac{-\frac{|\Sigma_k|^2}{4N^2}}{2|\Sigma_k|+\frac{2}{3}\cdot2\cdot\frac{|\Sigma_k|}{2N}}\right)
\end{align*}
Since the right-hand side of the above inequality does not depend on $\delta_i$ and $\delta_j$, we can bound 
$$
\Prob(\F \text{ is bad}) \leq \binom{N!}{2}\cdot d_\rho^2 \cdot \exp\left(\frac{-\frac{|\Sigma_k|^2}{4N^2}}{2|\Sigma_k|+\frac{2}{3}\cdot2\cdot\frac{|\Sigma_k|}{2N}}\right)
$$
Substituting $d_{\rho} = N-1$ and $|\Sigma_k| = k\binom{N}{2}$, and after some algebraic manipulations, we find that taking 
$$k \geq \frac{16(3N + 1)}{3(N - 1)} \left[ N \ln N + \ln \binom{N!}{2} + 2\ln(N - 1) \right]$$
guarantees that the probability of $\F$ being bad is upper bounded by $\exp(-N\ln N)$, concluding the proof.
\end{proof}

\probt*

\begin{proof}
Let $\delta_i, \delta_j \in \F$ with $i \neq j$. By \Cref{thm:p_agree} we have 
\begin{align*}
P_{\operatorname{agree}} &= \frac{1}{N} + \frac{1}{N} \vect{v}^\top M_{\Pi_0}^T \vect{v}
\end{align*}
where $\vect{v} = \sum_{i=1}^{N-1} \vect{e_i} \otimes \vect{e_i}$. By \Cref{thm:spectral_bound_random}, for the choice of $k$ assumed, with probability $1-\exp(-N\ln N)$ we have 
$$\left|P_{\operatorname{agree}}-\frac{1}{N}\right|=\frac{|\vect{v}^\top M_{\Pi_0}^T \vect{v}|}{N} \leq \frac{\|M_{\Pi_0}^T\|_2 \cdot \|\vect{v}\|^2}{N} \leq \left(1-\frac{1}{2N}\right)^T \frac{N-1}{N} \leq \left(1-\frac{1}{2N}\right)^T.$$
This concludes the proof.
\end{proof}

\randt*

\begin{proof}
From the bound of \Cref{lem:prob_of_agreement}, the inequality $1-x \le e^{-x}$ which holds for all $x\in \RR$, and the choice of $T \geq 2N \ln (N!)$, we get:
\[
\left|P_{\operatorname{agree}}-\frac{1}{N}\right| \leq \left(1-\frac{1}{2N}\right)^T \leq \exp\left(-\frac{T}{2N}\right) \leq \exp\left(-\frac{2N\ln(N!)}{2N}\right) = \frac{1}{N!},
\]
concluding the proof.
\end{proof}
\section{Proof of \Cref{thm:rand_lb}}
\label{app:proof_rand_lb}

In this section, we prove \Cref{thm:rand_lb}, which we restate for convenience:

\randlb*

The proof requires deriving a lower bound on $|P_{\operatorname{agree}}-1/N|$ that decays exponentially in $T$, which is given by the following lemma:

\begin{lemma}
\label{thm:err_lb}
For $M=N!$ and $N\ge 5$, if we choose the alphabet parameter 
$$k \geq \frac{16(3N + 1)}{3(N - 1)} \left[ N \ln N + \ln \binom{N!}{2} + 2\ln(N - 1) \right],$$ then any pair of semiautomata $(\delta_i, \delta_j)$ from a randomized $(k, M)$-shuffle family $\mathcal{F}$ satisfies
\[\left| P_{\operatorname{agree}} - \frac{1}{N}\right| \geq \frac{1}{2} \left(1 - \frac{3}{N}\right)^{T} \]
with probability at least $1 - \exp(-N\ln N)$.
\end{lemma}
\begin{proof}
    From \Cref{eq:expected_spec}, the minimum eigenvalue of the expected operator $\E[M_{\Pi_0}^{i,j}]$ is given by $\lambda_{\min}(\E[M_{\Pi_0}^{i,j}]) = \frac{N-3}{N-1}$. Furthermore, we have shown that, for $k$ chosen as in the statement, with probability at least $1-\exp(-N\ln N)$, the deviation $\|M_{\Pi_0}^{i,j} - \E[M_{\Pi_0}^{i,j}]\|_2$ satisfies
    $$\|M_{\Pi_0}^{i,j} - \E[M_{\Pi_0}^{i,j}]\|_2 \leq \frac{1}{2N}.$$
    By Weyl's inequality, with probability at least $1-\exp(-N\ln N)$, it holds
    \begin{equation}
    \label{eq:min_eig_lb}
    \lambda_{\min}(M_{\Pi_0}) \geq \lambda_{\min}(\E[M_{\Pi_0}^{i,j}]) - \|M_{\Pi_0}^{i,j} - \E[M_{\Pi_0}^{i,j}]\|_2 \geq \frac{N-3}{N-1} - \frac{1}{2N}.
    \end{equation}
    For $N \geq 4$, the right-hand side of the above inequality is positive, and so $M_{\Pi_0}^{i,j}$ is positive definite, in which case 
    $$\vect{v}^\top \left(M_{\Pi}^{i,j}\right)^T \vect{v} \geq \left(\lambda_{\min}(M_{\Pi}^{i,j})\right)^T \cdot \|\vect{v}\|^2$$
    Substituting $\|\vect{v}\|^2 = N-1$ and the lower bound for the minimum eigenvalue obtained in \Cref{eq:min_eig_lb}, we obtain 
    \begin{align*}
        \left|P_{\operatorname{agree}} - \frac{1}{N}\right| &= \frac{\vect{v}^\top \left(M_{\Pi}^{i,j}\right)^T \vect{v}^\top}{N} \geq \frac{N-1}{N}\left(\frac{N-3}{N-1}-\frac{1}{2N}\right)^T \geq \frac{1}{2}\left(1-\frac{3}{N}\right)^T,
    \end{align*}
    where the last inequality is valid for $N\geq 5$. This concludes the proof.
\end{proof}

Using the lower bound established in \Cref{thm:err_lb}, we are now ready to prove \Cref{thm:rand_lb}:

\begin{proof}[Proof of \Cref{thm:rand_lb}]
    By \Cref{thm:err_lb}, $T$ must satisfy
    $$\frac{1}{2}\left(1-\frac{3}{N}\right)^T \leq \frac{1}{N!}.$$
    Solving for $T$ we obtain 
    $$T \geq \frac{\ln(N!/2)}{\ln\left(\frac{N}{N-3}\right)}.$$
    The numerator is $\Theta(N\ln N)$ while a Taylor approximation on the denominator shows that for $N \gg 1$: 
    $$\ln\left(\frac{N}{N-3}\right) = -\ln\left(1-\frac{3}{N}\right) \approx \frac{3}{N}.$$
    Thus, the denominator is $\Theta(1/N)$, which shows that $T = \Omega(N^2 \ln N)$.
\end{proof}
\section{Proof of \Cref{thm:sq_lower_bound}}
\label{app:proof_sqlb}

In this section, we prove \Cref{thm:sq_lower_bound}, by generalizing the standard argument for the case $\Y = \{0,1\}$ (e.g. Theorem 2 in \cite{characterizing}). For convenience, we restate the theorem first:

\sqlb*

\begin{proof}
    We represent each concept $f \in \C$ as a vector-valued function $\vect{u}_f : \X \to \RR^{|\Y|}$ given by $\vect{u}_f(x) = \vect{e}_{f(x)} - \bar{\vect{e}}$ where $\vect{e}_y$ is the standard basis vector corresponding to label $y$ and $\bar{\vect{e}}$ is the vector with all entries equal to $1/|\Y|$. It is easy to verify that by the linearity of expectation, for any $f, g \in \C$ we have $\chi(f, g) = \langle\vect{u}_f, \vect{u}_g\rangle_D$ where the $L^2(\X, \RR^{|\Y|})$ inner product $\langle \cdot, \cdot\rangle_D$ is defined as
    $$\langle \vect{u}, \vect{v}  \rangle_D := \mathbb{E}_{x\sim D}\left[\langle \vect{u}(x), \vect{v}(x)\rangle\right].$$
    Assume that $f_1,\dots,f_d \in \C$ fulfill $|\chi(f_i, f_j)| \leq 1/d$ for all $i, j\in [d]$ with $i \neq j$. To discharge notation, we will write $\vect{u}_i$ instead of $\vect{u}_{f_i}$ to refer to the representation defined above. By the previous discussion, we have $|\langle\vect{u}_i, \vect{u}_j\rangle_D| \leq 1/d$ for all $i, j \in [d]$ with $i\neq j$. We present an adversarial answering strategy for the oracle that guarantees that the learner can eliminate only a small number of concepts after every answer when the ground-truth concept $f^*$ is one of the $f_i$'s.

    Let $h: \X \times \Y \to [-1,1]$ be an arbitrary query function used by the learner. The learner requests the value of the expectation:
    $$v^* := \E_{x\sim D}\left[h(x, f^*(x))\right]$$
    and the oracle returns an answer $\hat{v}$ such that $|v^* - \hat{v}| \leq \tau$. Using this information, the learner can eliminate any $f_i$ for which $|v_i - \hat{v}| > \tau$.
    We represent $h$ by the vector-valued function $\vect{H}: \X \to \RR^{|\Y|}$ where the $y$-th component of $\vect{H}(x)$ is equal to $h(x, y)$. Since $h(x, f_i(x)) = \langle \vect{H}(x), \vect{e}_{f_i(x)} \rangle$, we have:
    $$v_i = \E_{x\sim D}[\langle \vect{H}(x), \vect{e}_{f_i(x)} \rangle] = \langle \vect{H}, \vect{e}_{f_i} \rangle_D.$$
    Since $\vect{u}_i(x) = \vect{e}_{f_i(x)} - \bar{\vect{e}}$, we can rearrange this to write $\vect{e}_{f_i(x)} = \vect{u}_i(x) + \bar{\vect{e}}$. Substituting this into the expression for $v_i$:
    \begin{align*}
        v_i &= \langle \vect{H}, \vect{u}_i + \bar{\vect{e}} \rangle_D = \langle \vect{H}, \vect{u}_i \rangle_D + \langle \vect{H}, \bar{\vect{e}} \rangle_D.
    \end{align*}
    Consider the adversarial answering strategy where the oracle responds with $\hat{v} = \langle \vect{H}, \bar{\vect{e}} \rangle_D$. As such, the learner eliminates all concepts $f_i$ for which 
    $\left| \langle \vect{H}, \vect{u}_i \rangle_D \right| > \tau$.    
    Under this answering strategy, we can count how many candidates the learner can eliminate with each query. Define $A^+ = \{i \in [d] : \langle \vect{H}, \vect{u}_i\rangle_D \geq \tau\}$ and $A^- = \{i\in [d]: \langle\vect{H}, \vect{u}_i\rangle_D \leq -\tau\}$, and notice that the number of eliminated candidates is precisely $|A^+|+|A^-|$. To upper bound the cardinality of $A^+$, we apply the Cauchy-Schwartz inequality to obtain:
    \begin{equation}
    \label{eq:sum_cs}
    \left\langle \vect{H}, \sum_{i\in A^+}\vect{u}_i \right\rangle_D^2 \leq \|\vect{H}\|_D^2 \cdot \left\| \sum_{i\in A^+} \vect{u}_i\right\|_D^2
    \end{equation}
    By the definition of $A^+$, the left-hand side of \Cref{eq:sum_cs} is at least $(|A^+|\tau)^2$. On the other hand, since $|h(x,y)| \leq 1$ we have
    $$\|\vect{H}\|_D^2 = \mathbb{E}_{x\sim D}\left[\sum_{y \in \Y} h(x, y)^2\right] \leq |\Y|,$$
    and 
    $$\left\| \sum_{i\in A^+} \vect{u}_i\right\|_D^2 = \sum_{i,j \in A^+}\langle\vect{u}_i, \vect{u}_j \rangle_D = \sum_{i\in A^+} \|\vect{u}_i\|_D^2 + \sum_{i\neq j} \langle\vect{u}_i, \vect{u}_j\rangle_D \leq |A^+|(1-1/|\Y) + \frac{|A^+|^2}{d}$$
    where in the last inequality we used the fact that $\|\vect{u}_i\|_D^2 = \chi(f_i, f_i) = 1-1/|\Y|$. Chaining the inequalities we get 
    $$(|A^+|\tau)^2 \leq |\Y|\left(|A^+|(1-1/|\Y) + \frac{|A^+|^2}{d}\right).$$
    Dividing by $|A^+|$ and rearranging yields:
    $$|A^+| \leq \frac{d(|\Y| -1)}{d\tau^2 -|\Y|}.$$

    A similar procedure for $A^-$ shows that the number of eliminated concepts after the query is at most $\frac{2d(|\Y|-1)}{d\tau^2 - |\Y|}$ if the adversary returns $\langle \vect{H}, \bar{\vect{e}}\rangle_D$. Thus, the learner requires at least $\frac{(d-1)(d\tau^2 -|\Y|)}{2d(|\Y|-1)}$ queries to identify the ground-truth concept $f^*$.
\end{proof}

\end{document}